\documentclass[
  aps,
  prresearch,
  reprint,
  amsmath,
  amssymb,
  superscriptaddress,
  nofootinbib,
  longbibliography
]{revtex4-2}
\usepackage{etex}
\usepackage{amsfonts,amsmath,amsthm,mathtools}
\usepackage{graphicx}
\usepackage[usenames,x11names,dvipsnames,svgnames]{xcolor}
\usepackage{dsfont}
\usepackage{dcolumn}
\usepackage{hyperref}
\usepackage{booktabs}
\usepackage{array}
\usepackage{multirow}
\usepackage{tikz}
\usepackage{pgfplots}
\pgfplotsset{compat=1.18}
\makeatletter
\providecommand{\href@noop}[2]{#2}

\makeatother
\hypersetup{
  colorlinks=true,
  linkcolor=blue,
  citecolor=blue, 
  urlcolor=blue
}

\newcommand{\mltlne}[2][\EQSP]{\begingroup\setlength\abovedisplayskip{#1}\setlength\belowdisplayskip{#1}\begin{equation}\begin{multlined} #2 \end{multlined}\end{equation}\endgroup\noindent}
\newcommand{\cgather}[2][\EQSP]{\begingroup\setlength\abovedisplayskip{#1}\setlength\belowdisplayskip{#1}\begin{gather} #2 \end{gather}\endgroup\noindent}
\newcommand{\cgathers}[2][\EQSP]{\begingroup\setlength\abovedisplayskip{#1}\setlength\belowdisplayskip{#1}\begin{gather*} #2 \end{gather*}\endgroup\noindent}
\newcommand{\calign}[2][\EQSP]{\begingroup\setlength\abovedisplayskip{#1}\setlength\belowdisplayskip{#1}\begin{align} #2 \end{align}\endgroup\noindent}
\newcommand{\caligns}[2][\EQSP]{\begingroup\setlength\abovedisplayskip{#1}\setlength\belowdisplayskip{#1}\begin{align*} #2 \end{align*}\endgroup\noindent}
 \newcommand{\tc}[2]{\begin{minipage}[t]{#1}\raggedright #2\end{minipage}} 
\newcolumntype{L}[1]{>{\raggedright\arraybackslash}p{#1}}
\newcolumntype{C}[1]{>{\centering\arraybackslash}p{#1}}
\newcolumntype{R}[1]{>{\raggedleft\arraybackslash}p{#1}}
\newtheorem{theorem}{Theorem}
\newtheorem{proposition}{Proposition}
\newtheorem{definition}{Definition}
\newtheorem{lemma}{Lemma}
\newtheorem{corollary}{Corollary}

\newtheorem{assumption}{Assumption}
\newtheorem{hypothesis}{Hypothesis}
\newcommand{\KL}{D_{\mathrm{KL}}}
\newcommand{\TV}{\mathrm{TV}}
\newcommand{\Prob}{\mathbb{P}}
\newcommand{\Exp}{\mathbb{E}}
\newcommand{\calI}{\mathcal{I}}

\newcommand{\dd}{\,d}

\begin{document}
\title{Thermodynamic Measure of Intelligence}

\author{Ishanu Chattopadhyay}
\email{ishanu\_ch@uky.edu}
\affiliation{Institute for Biomedical Informatics, University of Kentucky, Lexington, Kentucky, USA}
\affiliation{Department of Computer Science, University of Kentucky, Lexington, Kentucky, USA}

\date{\today}

\begin{abstract}
Can intelligence be measured? We propose that intelligence can be defined as the lawful amplification of rare but valid futures: a system increases the probability of outcomes that would be unlikely under passive dynamics but remain admissible under the constraints of the domain. We start with the premise  that an intelligent system must model the world and its own place within it. Because the system is part of the world it models, this leads naturally to recursive self-simulation: the system represents futures in which its own actions are part of the trajectory. Our central results give a necessity statement and a conditional near-sufficiency statement connecting this architecture to a precise thermodynamic measure of lawful amplification of rare-valid futures: high rare-valid lift is impossible unless the internal simulation identifies rare-valid futures with high fidelity; conversely, when rare-valid fidelity is high and the simulation contains an effective policy, the achievable lift approaches the actuation-limited optimum. Thus recursive self-simulation is not merely a plausible feature of intelligence but, under the stated assumptions, is necessary and nearly sufficient for high thermodynamic intelligence. The resulting framework makes intelligence measurable on a universal scale, from passive matter and feedback controllers, large language models, and  humans as text generators to Maxwell-demon-like information engines.
\end{abstract}

\maketitle

\allowdisplaybreaks
\section{Introduction}
\label{sec:introduction}

Can intelligence be framed as a measurable physical quantity? We start from the observation that any system perceived to be intelligent models the world with itself inside it, simulates possible futures conditioned on its own actions, and uses that internal model to make some futures more likely than they would be under passive dynamics. This suggests that a relevant measurable quantity is \emph{rare-valid lift}: the increase in probability assigned to futures that are unlikely under a passive baseline but remain valid under the constraints of the domain. Our central result is that high rare-valid lift cannot be obtained merely from randomness or strong actuation. Under bounded amplification, it requires high-fidelity self-simulation: the system's internal model must identify rare-valid futures accurately enough to target them.

\begin{figure*}[!t]

  \IfFileExists{figtom.pdf}{%
  \includegraphics[width=.85\textwidth]{figtom}%
  }{%
  \fbox{\parbox{.82\textwidth}{\centering Missing figure file \texttt{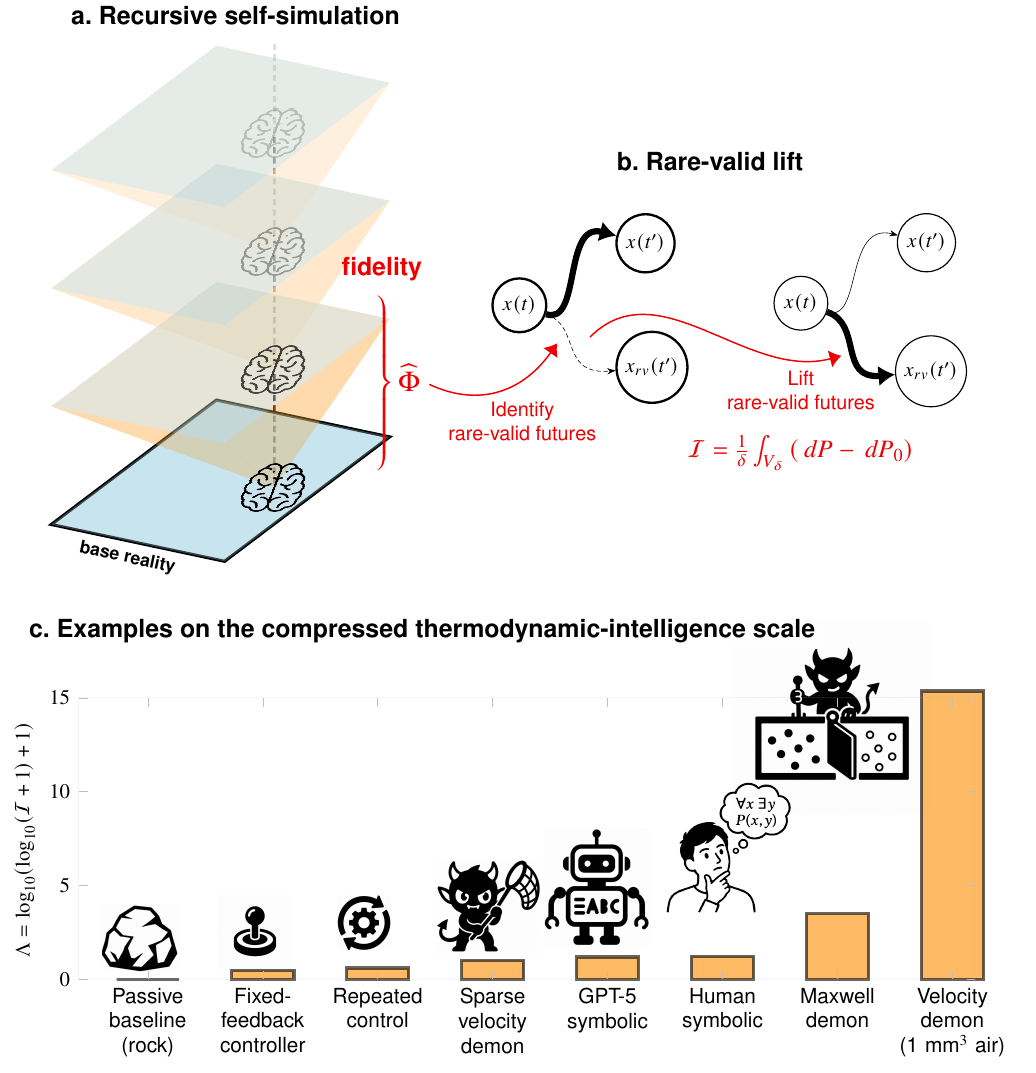}. Upload it with the arXiv source bundle.}}%
  }
\vspace{-5pt}
  
 \caption{
Conceptual summary. (a) Recursive self-simulation: a system represents the world at one level together with models of its own future states and actions at higher simulated levels. The rare-valid fidelity ($\widehat\Phi$) measures how accurately the simulation identifies targetable rare-valid futures. (b) Thermodynamic intelligence: relative to a passive trajectory law ($P_0$), an induced law ($P$) shifts probability mass toward rare-valid trajectories ($V_\delta$), producing rare-valid lift ($I_\delta=\delta^{-1}\int_{V_\delta}(dP-dP_0)$). (c) Representative systems on the compressed scale ($\Lambda=\log_{10}(\log_{10}(I+1)+1)$). The plotted examples are finite-resolution calibrations of probability lift, with symbolic and demon entries interpreted under the assumptions stated in the text.
}\label{figtom}
\vspace{-10pt}

\end{figure*}  

Standard definitions of intelligence emphasize behavior: imitation or conversational indistinguishability~\citep{turing1950computing}, learning, reasoning, planning, generalization, compression, reward maximization, or task success~\citep{legg2007universal,russellnorvig2021aima,chollet2019measure}. These criteria are useful, but they do not by themselves identify a substrate-independent operation common to brains, large language models, microbial communities, immune repertoires, controllers, and idealized information engines such as Maxwell's demons. All of these systems transform information into action. We therefore ask a different question: what does an intelligent system do to the likelihood of possible futures?

This question is related to, but distinct from, existing task-facing accounts of  intelligence. Legg--Hutter intelligence defines an agent's intelligence by its expected reward over a universal distribution of computable environments~\citep{legg2007universal}. Chollet's ARC framework instead emphasizes skill-acquisition efficiency: the ability to infer abstract rules and generalize from sparse experience under human-like priors~\citep{chollet2019measure}. These accounts evaluate performance across environments, benchmarks, or problem classes. Our  framework is path-facing. We  ask what physical or probabilistic operation underlies such performance once a level of description, baseline law, validity criterion, and observational resolution are fixed. Reward maximization, benchmark generalization, theorem proving, symbolic problem solving, and biological adaptation can then be treated as special cases.

We begin with recursive self-simulation. A system acts intelligently, in the present sense, when it carries a model of the world that includes itself as a causal object. Such a model can represent possible futures, including the system's own interventions, the observations those interventions may produce, and the way its own information state may later change. Minsky anticipated this point in his account of internal models: advanced problem solving requires a system to represent its own goals, resources, and problem-solving activity, and self-understanding can involve models of models of oneself~\citep{minsky1965matter,minsky1986society,minsky2006emotion}. Related ideas appear in work on self-reference and strange loops~\citep{godel1931undecidable,hofstadter1979geb,hofstadter2007strange}.

To measure the capabilities enabled by recursive self-simulation, we use path laws. Passive dynamics induce a baseline distribution \(P_0\) over trajectories. A controlled or agent-like system induces another law \(P\). A self-simulation becomes observable through the way it changes these probabilities. Intelligence, viewed this way, is lawful trajectory reweighting: some futures become more likely, others less so, and the change must respect the thermodynamic accounting required by measurement, memory, computation, control, and erasure.

Not all trajectory manipulations are the same. The most informative changes occur in the tail of the trajectory distribution. Moving probability among futures already common under \(P_0\) may reflect stabilization or regulation, but it does not strongly test whether the system can reach beyond passive dynamics. Rare futures probe that ability because they would otherwise remain effectively unrealized. Rarity alone, however, is not a measure of intelligence; random noise also produces improbable events. The futures must remain valid; and hence we focus on rare-valid futures: trajectories that have low probability under the passive law but remain admissible under the constraints of the domain. Rarity supplies counterfactual difficulty; validity, interpreted as physical realizability, biological viability, semantic coherence, executable correctness, or functional success, prevents the measure from rewarding noise.

To make this notion precise and computable, we consider the thermodynamics of system trajectories, and we define  thermodynamic intelligence as rare-valid probability lift: the fractional increase, under the induced law \(P\), in the probability of an exceedingly rare but valid set. The definition turns the quantification of intelligence into a question about path measures. The thermodynamic machinery needed for this analysis is well developed~\citep{landauer1961irreversibility,bennett1982thermodynamics,sagawa2010generalized,parrondo2015thermodynamics,mandal2012work}, including non-equilibrium fluctuation theorems, which quantify the relative likelihood of entropy-producing and entropy-reducing trajectories under passive dynamics~\citep{evans1993probability,jarzynski1997nonequilibrium,crooks1999entropy,seifert2005entropy}. The symbolic case uses the complementary information-theoretic language of entropy rate and coding~\citep{shannon1948mathematical,coverthomas2006elements}. We illustrate the framework through examples spanning passive systems with zero lift, simple controllers with modest amplification,  symbolic generators, including GPT-5 and human text, to Maxwell-demon-like information engines.
Maxwell's demon provides  the canonical historical case: a hypothetical microscopic observer using information about particle states to sort thermal fluctuations and create an apparent local entropy reduction. Here the demon serves as an idealized high-lift limit, where near-perfect microstate simulation, rare-valid identification, and actuation produce extreme rare-valid trajectory amplification before implementation costs are paid. Conversely, the symbolic examples show how the same formalism can be applied at a finite linguistic resolution, once the baseline ensemble, validity criterion, sequence length, and generator-induced probability shift are specified.

\section{Recursive Self-Simulation}
\label{sec:selfmodeling}

To act with intelligence, a system needs a deep model of its world, with itself in it. Minsky anticipated this point in his account of internal models: advanced problem solving requires a system to represent not only the external situation, but also its own goals, resources, and problem-solving activity, and self-understanding can involve models of models of oneself~\citep{minsky1965matter,minsky1986society,minsky2006emotion}. We formalize this self-in-world requirement as an embedded representation hierarchy. Let \(B\) denote an agent-like system embedded in an environment \(E\), and let \(U=B\cup E\). At a minimal level, the system maintains an internal representation of its observable world,
\cgather{
    r_B^{(0)}=r_B^{(0)}(U).
}
Because \(B\in U\), a sufficiently general model of the local universe must also represent the system itself: its state, memory, uncertainty, actions, and possible future updates, inducing a recursive hierarchy:
\calign{
    r_B^{(1)} &= r_B^{(1)}(r_B^{(0)}),\\
    r_B^{(2)} &= r_B^{(2)}(r_B^{(1)}),\\
    &\vdots \nonumber\\
    r_B^{(k)} &= r_B^{(k)}(r_B^{(k-1)}).
}
The hierarchy encodes predictions about the world, predictions about the system's own future actions, and predictions about how its information state may change; recursive self-simulation is therefore a finite self-referential loop. If the environment contains other agents \(B_j\), the same idea extends to nested representations:
\cgather{
    r_B^{(k)}\!\left(r_{B_j}^{(\ell)}\right),
    \qquad j\neq B,\quad k,\ell\geq0.
}%
Social and biological environments can therefore generate interacting recursive models. This is the thermodynamic analogue of theory-of-mind style modeling: an agent's future depends partly on what it predicts other agents will perceive, infer, and do~\citep{premack1978chimpanzee,dennett1987intentional}.

Recursive self-simulation becomes operational when a system evaluates consequences of its own future actions. Let \(h_t\) denote the history available at time \(t\), let \(\mathcal A_t\) denote the available action set, and let \(\Gamma_{t:T}\) denote a future trajectory segment. Suppose an agent evaluates a trajectory functional \(G\) under its \(k\)-level internal model and selects
\cgather{
    a_t^\star
    \in
    \arg\max_{a\in\mathcal A_t}
    \Exp_{r_B^{(k)}}\!\left[
    G(\Gamma_{t:T})\mid h_t,a
    \right].
    \label{eq:recursive_action_choice}
}%
Equivalently, the agent may implement a policy \(\pi_t(\cdot\mid h_t)\) concentrated near such maximizing actions. To evaluate this expectation, the model must represent the future environment, the agent's possible actions, and the agent's own future information state. Our construction is distinct from predictive-processing and active-inference views of perception and action~\citep{friston2010free}; our measured quantity is rare-valid probability lift relative to \(P_0\), not free energy itself.

The hierarchy of recursive simulation determines which futures the system can identify, evaluate, and target. To connect this architecture to a measurable quantity, we now introduce the level-relative rare-valid lift. Let \(\mathcal L_0\) denote the base reality, or physical level, under consideration. Throughout the paper, \(\calI\) denotes a fractional probability lift of a rare-valid set relative to a passive baseline. At level \(k\), this lift compares the probability assigned to a rare-valid event under an induced or simulated law with its probability under the corresponding passive law. Realized intelligence is the lift actually induced in the level-\(k\) path law. Intelligence potential is the largest such lift available inside a level-\((k+1)\) simulation of level \(k\). Section~\ref{sec:definition} gives the corresponding trajectory-space definition and identifies it as thermodynamic intelligence.

For \(k\geq0\), let \(\Omega_k\) be the trajectory space at level \(k\), let \(\mathcal F_k\) be its \(\sigma\)-algebra, let \(P_{0,k}\) be the passive or baseline law, and let \(\eta_k\) be the observational resolution. Let \(V_{\delta,k}\subseteq\Omega_k\) be a measurable rare-valid event at level \(k\) with baseline mass
\cgather{
    \delta_k \triangleq P_{0,k}(V_{\delta,k})>0.
    \label{eq:level_delta_k}
}
When the target mass is fixed or clear from context, we write \(\delta\) rather than \(\delta_k\). Define
\cgather{
    \mathcal L_k\triangleq (\Omega_k,\mathcal F_k,P_{0,k},V_{\delta,k},\eta_k).
}
For any level-\(k\) path law \(Q_k\) satisfying \(V_{\delta,k}\in\mathcal F_k\), define the level-\(k\) rare-valid lift
\mltlne{
    \calI_{\delta,k}(Q_k;P_{0,k},V_{\delta,k})
    \triangleq
    \frac{Q_k(V_{\delta,k})-P_{0,k}(V_{\delta,k})}
    {P_{0,k}(V_{\delta,k})}\\
    =
    \frac{Q_k(V_{\delta,k})-\delta_k}{\delta_k}.
    \label{eq:level_relative_lift}
}%
The hierarchy is recursive in the upward direction: \(\mathcal L_{k+1}\) is a simulation or model of \(\mathcal L_k\). We write hatted quantities for the representation of level \(k\) inside level \(k+1\):
\cgathers{
    \widehat{\mathcal L}_{k+1\to k}
    =
    (\widehat\Omega_{k+1\to k},
    \widehat{\mathcal F}_{k+1\to k},
    \widehat P_{0,k+1\to k},
    \widehat V_{\delta,k+1\to k},
    \widehat\eta_{k+1\to k}).
}
Here \(\widehat V_{\delta,k+1\to k}\) is the level-\((k+1)\) representation of the level-\(k\) rare-valid set, with simulated baseline mass
\cgather{
    \widehat\delta_{k+1\to k}
    \triangleq 
    \widehat P_{0,k+1\to k}(\widehat V_{\delta,k+1\to k})>0.
    \label{eq:hat_delta}
}
When no ambiguity is possible, write \(\widehat P_0\), \(\widehat P_\pi\), \(\widehat V_\delta\), and \(\widehat\delta\) for the corresponding level-\((k+1\to k)\) quantities. The realized rare-valid lift of \(B\) at level \(k\), when the system induces the actual level-\(k\) path law \(P_{B,k}\), is
\mltlne{
    \calI^{\rm real}_{\delta,k}(B)
    \triangleq 
    \calI_{\delta,k}(P_{B,k};P_{0,k},V_{\delta,k})\\
    =
    \frac{P_{B,k}(V_{\delta,k})-P_{0,k}(V_{\delta,k})}
    {P_{0,k}(V_{\delta,k})}.
    \label{eq:realized_level_intelligence}
}
If \(P_{B,k}=P_{0,k}\), then \(\calI^{\rm real}_{\delta,k}(B)=0\) relative to that baseline. This does not say the system lacks intelligence; it says that intelligence is not realized as a path-law change at that level.

The intelligence potential for level \(k\) is computed inside the level-\((k+1)\) simulation of level \(k\). Let \(\widehat\Pi_{k+1\to k}\) be the simulated policy class. For \(\pi\in\widehat\Pi_{k+1\to k}\), define the simulated rare-valid lift
\cgather{
    \widehat{\calI}_{\delta}^{(k+1\to k)}(\pi)
    \triangleq
    \frac{
    \widehat P_{\pi,k+1\to k}(\widehat V_\delta)
    -
    \widehat\delta
    }
    {\widehat\delta}.
    \label{eq:potential_intelligence_policy}
}
The intelligence potential of \(B\) for level \(k\), as represented at level \(k+1\), is the supremal simulated rare-valid lift over the policy class available in that representation:
\cgather{
    \calI^{\rm pot}_{\delta,k}(B)
    \triangleq 
    \sup_{\pi\in\widehat\Pi_{k+1\to k}}
    \widehat{\calI}_{\delta}^{(k+1\to k)}(\pi).
    \label{eq:potential_intelligence_sup}
}
Thus actuation is not assumed at level \(k\) when potential is computed. The simulated action variables live in \(\mathcal L_{k+1}\); realization at level \(k\) is the separate question of whether a simulated policy can indeed be  implemented as an actual path-law change.

\paragraph{Rare-valid simulation fidelity}
\label{sec:rv_fidelity}

The relevant fidelity is not generic prediction accuracy. A model may predict common trajectories well while missing the rare-valid set, or be coarse in irrelevant coordinates while accurate on the rare-valid futures and actions that matter. We therefore define fidelity directly on the target set.

Let \(\widehat A_{k+1\to k}\subseteq\widehat\Omega_{k+1\to k}\) denote the set of trajectories identified by the level-\((k+1)\) simulation as targetable rare-valid futures for level \(k\). We define the level-specific rare-valid self-simulation fidelity
\cgather{
    \widehat\Phi_{k+1\to k}
    \triangleq 
    \frac{\widehat P_{0,k+1\to k}(\widehat A_{k+1\to k}\cap\widehat V_{\delta,k+1\to k})}{\widehat\delta_{k+1\to k}}.
    \label{eq:rv_fidelity}
}%
When the represented level is fixed, we write \(\widehat\Phi\) for \(\widehat\Phi_{k+1\to k}\). The set \(\widehat V_\delta\) is ``true'' only relative to the specified level-\(k\) description and its represented validity criterion. Thus \(\widehat\Phi=1\) means that, at the simulated baseline resolution, the targetable set covers the represented rare-valid set; \(\widehat\Phi=0\) means that the simulation misses it. The corresponding rare-valid self-simulation error is
\cgather{
    \widehat\varepsilon^{\rm RV}_{k+1\to k}
    \triangleq 
    1-\widehat\Phi.
    \label{eq:rv_error}
}%
Next we have our necessity result: under bounded amplification, high level-relative intelligence potential cannot be obtained from low rare-valid simulation fidelity.
\begin{theorem}[Rare-valid self-simulation fidelity is necessary]
\label{thm:rv_fidelity_necessity}
Work inside the level-\((k+1)\) simulation of level \(k\), and abbreviate
\(\widehat P_0=\widehat P_{0,k+1\to k}\),
\(\widehat P_\pi=\widehat P_{\pi,k+1\to k}\),
\(\widehat V_\delta=\widehat V_{\delta,k+1\to k}\),
\(\widehat A=\widehat A_{k+1\to k}\),
\(\widehat\delta=\widehat P_0(\widehat V_\delta)>0\), and let
\(\widehat\Phi\) denote the fidelity in Eq.~\eqref{eq:rv_fidelity}.
Assume \(\widehat P_\pi\ll \widehat P_0\). Suppose there exists \(\alpha_{\max}\geq1\) such that
\cgather{
    \frac{d\widehat P_\pi}{d\widehat P_0}(\omega)
    \leq
    \alpha_{\max}
    \quad
    \widehat P_0\text{-a.e. on }\widehat A\cap\widehat V_\delta,
    \label{eq:alpha_max_on_true_positive}
}
and
\cgather{
    \frac{d\widehat P_\pi}{d\widehat P_0}(\omega)
    \leq
    1
    \quad
    \widehat P_0\text{-a.e. on }\widehat V_\delta\setminus\widehat A.
    \label{eq:no_unidentified_amplification}
}
Then
\cgather{
    \widehat{\calI}_{\delta}^{(k+1\to k)}(\pi)
    \leq
    (\alpha_{\max}-1)\widehat\Phi.
    \label{eq:rv_necessity_bound}
}
And, if \(\alpha_{\max}>1\) and for some \(I_0>0\),
\(\widehat{\calI}_{\delta}^{(k+1\to k)}(\pi)\geq I_0\), then
\cgather{
    \widehat\Phi
    \geq
    \frac{I_0}{\alpha_{\max}-1}.
    \label{eq:rv_necessity_i0}
}
Thus high intelligence potential requires high rare-valid simulation fidelity relative to the available amplification budget. In particular, if \(I_0>\alpha_{\max}-1\), no policy satisfying the amplification bound can attain \(I_0\).
\end{theorem}
\noindent\emph{Proof.} See Appendix~\ref{app:proofs}.

The near-converse requires an implementation assumption: the simulation must contain a policy that amplifies the correctly identified rare-valid region.

\begin{theorem}[Near-sufficiency under effective simulated actuation]
\label{thm:rv_fidelity_sufficiency}
Use the notation of Theorem~\ref{thm:rv_fidelity_necessity}. Suppose there exists a simulated policy \(\pi\) and constants \(\alpha_{\min}>1\), \(0\leq\beta_{\min}\leq\alpha_{\min}\) such that
\cgather{
    \frac{d\widehat P_\pi}{d\widehat P_0}(\omega)
    \geq
    \alpha_{\min}
    \quad
    \widehat P_0\text{-a.e. on }\widehat A\cap\widehat V_\delta,
    \label{eq:alpha_min_on_true_positive}
}
and
\cgather{
    \frac{d\widehat P_\pi}{d\widehat P_0}(\omega)
    \geq
    \beta_{\min}
    \quad
    \widehat P_0\text{-a.e. on }\widehat V_\delta\setminus\widehat A.
    \label{eq:beta_min_on_missed_valid}
}
Then
\cgather{
    \widehat{\calI}_{\delta}^{(k+1\to k)}(\pi)
    \geq
    \alpha_{\min}\widehat\Phi
    +
    \beta_{\min}(1-\widehat\Phi)
    -1.
    \label{eq:rv_sufficiency_bound_general}
}
In particular, if \(0\leq\varepsilon\leq1\) and \(\widehat\Phi\geq1-\varepsilon\), then
\cgather{
    \widehat{\calI}_{\delta}^{(k+1\to k)}(\pi)
    \geq
    (\alpha_{\min}-1)
    -
    (\alpha_{\min}-\beta_{\min})\varepsilon.
    \label{eq:rv_sufficiency_bound_epsilon}
}
Therefore, as \(\widehat\Phi\to1\), effective simulated actuation drives the intelligence potential toward the actuation-limited value \(\alpha_{\min}-1\).
\end{theorem}
\noindent\emph{Proof.} See Appendix~\ref{app:proofs}.

Theorems~\ref{thm:rv_fidelity_necessity} and~\ref{thm:rv_fidelity_sufficiency} are the formal bridge between recursive self-simulation and thermodynamic intelligence. Low rare-valid fidelity caps the achievable lift; high rare-valid fidelity, together with a policy that amplifies the correctly identified region, yields high lift. These statements concern intelligence potential for level \(k\) as computed in the level-\((k+1)\) simulation. Realized intelligence at level \(k\) additionally requires implementation as an actual level-\(k\) path-law change.

\section{Thermodynamic Intelligence}
\label{sec:definition}
The previous section described recursive self-simulation as the internal architecture. We now define the observable: probability lift over rare-valid regions of trajectory space. Because trajectory spaces may be continuous or high-dimensional, rarity is defined at finite observational resolution.

\begin{definition}[Rare-valid set at finite resolution] Let \(V\subset\Omega\) denote the set of valid trajectories. Validity is domain-dependent. In a physical system, validity means physical admissibility. In a biological system, it may mean viability or functional organization. In a symbolic system, it may mean grammaticality, semantic coherence, factual consistency, and task relevance. Let \(\Pi_\eta\) be a finite measurable partition of \(\Omega\) at observational resolution \(\eta\). For a cell \(C\in\Pi_\eta\), \(P_0(C)\) is the passive probability of observing a trajectory in that cell. A rare-valid set \(V_{\delta,\eta}\) is a union of valid cells with small passive probability and target passive mass \(\delta\). When exact normalization is possible, we choose \begin{equation} P_0(V_{\delta,\eta})=\delta. \end{equation} For a finite partition, exact equality need not hold for every \(\delta\). In that case one may either choose an attainable value of \(\delta\), use \(P_0(V_{\delta,\eta})\le \delta\), or obtain exact normalization by randomized inclusion of a boundary cell. Equivalently, \(V_{\delta,\eta}\) may be taken as the lowest-baseline-probability valid region of total passive mass \(\delta\), up to this boundary convention. When the resolution \(\eta\) is fixed, we write \(V_\delta\) for \(V_{\delta,\eta}\). \end{definition}

\begin{definition}[Thermodynamic intelligence at resolution \(\delta\)]
Let \(P\) be the trajectory distribution induced by a system, and let \(V_\delta\) satisfy \(P_0(V_\delta)=\delta\). Define the \(\delta\)-scale thermodynamic intelligence of \(P\) relative to \(P_0\) and \(V_\delta\) as
\begin{equation}
    \calI_\delta(P;P_0,V_\delta)
    =
    \frac{P(V_\delta)-P_0(V_\delta)}{\delta}
    =
    \frac{P(V_\delta)-\delta}{\delta}.
    \label{eq:Idelta}
\end{equation}
Equivalently, in density notation,
\begin{equation}
    \calI_\delta(P;P_0,V_\delta)
    =
    \frac{1}{\delta}
    \int_{V_\delta}
    \left(\dd P-\dd P_0\right).
\end{equation}
When the limit exists, define
\begin{equation}
    \calI(P;P_0,V)
    =
    \lim_{\delta\to 0^+}
    \calI_\delta(P;P_0,V_\delta).
\end{equation}
\end{definition}

If \(P=P_0\), then \(\calI_\delta=0\). If a system deterministically realizes a rare-valid cell in \(V_\delta\), so that \(P(V_\delta)=1\), then
\begin{equation}
    \calI_\delta
    =
    \frac{1-\delta}{\delta}
    \approx
    \frac{1}{\delta}.
\end{equation}
Thus the measure ranges naturally from zero for passive systems to very large values for ideal information engines that select extremely rare valid trajectories. The definition credits only probability mass moved into futures that are both low-probability under \(P_0\) and valid under the domain constraints. In flexible settings, recursive self-simulation should improve both the identification of such futures and the actions that make them more likely. Lemma~\ref{lem:rare_valid_binary_kl} records the information-theoretic consequence: amplifying a rare-valid set requires path-measure divergence from the passive baseline.
\begin{lemma}[Rare-valid amplification implies path-measure divergence]
\label{lem:rare_valid_binary_kl}
Let \(P_0(V_\delta)=\delta\), let \(p=P(V_\delta)\), and suppose \(0<\delta<1\). Define
\begin{equation}
    \calI_\delta
    =
    \frac{p-\delta}{\delta}.
\end{equation}
Then
\begin{equation}
    \KL(P\,\|\,P_0)
    \geq
    d(p\,\|\,\delta),
    \label{eq:binary_kl_bridge}
\end{equation}
where
\begin{equation}
    d(p\,\|\,\delta)
    =
    p\log\frac{p}{\delta}
    +
    (1-p)\log\frac{1-p}{1-\delta}
\end{equation}
is the binary KL divergence. Equivalently, when \(p=\delta(1+\calI_\delta)\leq1\),
\begin{equation}
    \KL(P\,\|\,P_0)
    \geq
    d\!\left(\delta(1+\calI_\delta)\,\|\,\delta\right).
    \label{eq:Idelta_kl_bridge}
\end{equation}
\end{lemma}
\noindent\emph{Proof.} See Appendix~\ref{app:proofs}.

\subsection{Trajectory-Space Thermodynamics}
\label{sec:thermo}

Rare-valid lift is a change in trajectory probabilities, so its natural thermodynamic setting is path space. Let \((\Omega,\mathcal F)\) denote a measurable space of trajectories \(\omega\) over \([0,\tau]\). Let \(P_0\) be the passive path measure and \(P_B\) the path measure induced by an agent \(B\) acting through feedback. Let \(S(\omega)\) denote physical entropy production along \(\omega\), and define the dimensionless entropy production
\cgather{
    \sigma(\omega)=\frac{S(\omega)}{k_B}.
}
For matched entropy-production bins \(A_s^+\) and \(A_s^-\), where \(A_s^+\) collects trajectories with \(\sigma(\omega)\approx+s\) and \(A_s^-\) collects the corresponding time-reversed or otherwise matched trajectories with \(\sigma(\omega)\approx-s\), the passive fluctuation-theorem relation is written at event level as
\cgather{
    \log\frac{P_0(A_s^+)}{P_0(A_s^-)}\simeq s.
    \label{eq:passive_event_relation}
}%
Equivalently, in dimensional units, if the matched bins correspond to entropy productions \(+\Delta S\) and \(-\Delta S\), then the right-hand side is \(\Delta S/k_B\). The symbol \(\simeq\) marks event-level coarse-graining: exact equality requires bins and matching conventions that preserve the underlying trajectory-level fluctuation relation.

Now, feedback replaces the passive law \(P_0\) by the controlled law \(P_B\). Microscopic feedback fluctuation relations generally contain trajectory-dependent measurement and information terms, so a universal scalar correction need not exist after coarse-graining. For the fixed bins used here, we therefore define the event-level information correction directly:
\cgather{
    J_s(B)
    \triangleq 
    s-
    \log\frac{P_B(A_s^+)}{P_B(A_s^-)}.
    \label{eq:Js_def}
}
Equivalently,
\cgather{
    \log\frac{P_B(A_s^+)}{P_B(A_s^-)}
    =
    s-J_s(B).
    \label{eq:controlled_event_relation}
}
Thus \(J_s(B)\) records how feedback changes the entropy-production log-ratio on the chosen bins relative to the passive fluctuation-theorem scale. It is a coarse-grained diagnostic, not a complete thermodynamic balance; measurement, memory, computation, control, and erasure remain part of the full physical accounting.

\paragraph{Coarse-grained path-deviation stability}
\label{sec:theorem1}

We next record a stability bound for the entropy-bin signatures just defined. Let \(P\) and \(Q\) be path measures on \((\Omega,\mathcal F)\). For matched entropy-production events \(A_s^+\) and \(A_s^-\), define
\cgather{
    \Delta_s(P,Q)
    \triangleq
    \log \frac{P(A_s^+)}{P(A_s^-)}
    -
    \log \frac{Q(A_s^+)}{Q(A_s^-)}.
    \label{eq:delta_pq}
}
For \(P=P_B\) and \(Q=P_0\), write \(\Delta_s(B)=\Delta_s(P_B,P_0)\). If the passive fluctuation relation holds in the dimensionless convention, then
\cgather{
    \Delta_s(B)
    =
    \log\frac{P_B(A_s^+)}{P_B(A_s^-)}
    -
    s,
}
with \(s\) replaced by \(\Delta S/k_B\) in dimensional units.

\begin{assumption}[Nondegenerate entropy bins]
\label{ass:nondegenerate_bins}
For the two path measures being compared, there exists \(m_s>0\) such that
\cgather{
    P(A_s^+),\,P(A_s^-),\,Q(A_s^+),\,Q(A_s^-)
    \geq m_s.
    \label{eq:nondegenerate_bins}
}
\end{assumption}

\begin{theorem}[Coarse-grained path-deviation bound]
\label{thm:path_deviation}
Under Assumption~\ref{ass:nondegenerate_bins},
\cgather{
    |\Delta_s(P,Q)|
    \leq
    \frac{\sqrt{2}}{m_s}
    \sqrt{\KL(P\,\|\,Q)}.
    \label{eq:path_deviation_bound}
}
In particular,
\cgather{
    |\Delta_s(B)|
    \leq
    \frac{\sqrt{2}}{m_s}
    \sqrt{\KL(P_B\,\|\,P_0)}.
    \label{eq:controlled_passive_deviation_bound}
}
\end{theorem}
\noindent\emph{Proof.} See Appendix~\ref{app:proofs}.

Theorem~\ref{thm:path_deviation} gives the thermodynamic role of path-measure divergence. Lemma~\ref{lem:rare_valid_binary_kl} shows that rare-valid amplification requires divergence from the passive law. Theorem~\ref{thm:path_deviation} shows that such divergence also controls how much coarse-grained entropy-production log-ratios can change on fixed nondegenerate bins. Thus the rare-valid lift is not an isolated score: when a controller reweights trajectory probabilities, the induced change is constrained in the same path-measure geometry that governs coarse-grained thermodynamic signatures. The bound is intentionally finite-bin and moderate-event; rare-event amplification itself is handled by Lemma~\ref{lem:rare_valid_binary_kl}.

\paragraph{Auxiliary model-to-control continuity}
\label{sec:model_to_control_auxiliary}

The fidelity theorems above are rare-set results. A separate continuity statement compares coarse-grained thermodynamic signatures induced by nearby path laws. It does not prove high thermodynamic intelligence from high fidelity; it only says that finite-depth controlled laws inherit the entropy-bin signatures of an ideal controlled law when the induced path laws are close.

Let \(P_B^{(k)}\) denote the controlled path law induced by a policy computed from the \(k\)-level recursive internal model \(r_B^{(k)}\). Let \(P_B^\star\) denote the ideal controlled path law induced by the limiting or perfectly faithful recursive model for the same objective and admissible control class. Let \(\varepsilon_k\geq0\) denote intervention-relevant prediction error.

\begin{assumption}[Model-to-control stability]
\label{ass:model_to_control}
There exist a constant \(C>0\) and a modulus \(\rho\), with \(\rho(\varepsilon)\to0\) as \(\varepsilon\to0\), such that
\cgather{
    D_{\mathrm{KL}}
    \left(
    P_B^{(k)}\,\|\,P_B^\star
    \right)
    \leq
    C\rho(\varepsilon_k).
    \label{eq:model_to_control_stability}
}
\end{assumption}

\begin{proposition}[Recursive fidelity controls convergence to the ideal thermodynamic signature]
\label{prop:recursive_signature_convergence}
Assume model-to-control stability. Suppose the entropy bins are nondegenerate under \(P_B^{(k)}\) and \(P_B^\star\), with lower bound \(m_s>0\). Then
\cgather{
    \left|
    \Delta_s(P_B^{(k)},P_0)
    -
    \Delta_s(P_B^\star,P_0)
    \right|
    \leq
    \frac{\sqrt{2C}}{m_s}
    \sqrt{\rho(\varepsilon_k)}.
    \label{eq:recursive_signature_convergence}
}
Consequently, if \(\varepsilon_k\to0\), then the finite-depth thermodynamic signature converges to the ideal controlled thermodynamic signature at the rate determined by \(\rho\).
\end{proposition}
\noindent\emph{Proof.} See Appendix~\ref{app:proofs}.

\subsection{Imperfect Rare-Set Identification}
\label{sec:theorem2}
The definition above assumes access to the true rare-valid set \(V_\delta\). Real agents infer an estimated set \(V'_\delta\), so amplification can be spent on false-positive trajectories that are rare but not valid. We model this protocol-level bookkeeping penalty.

\paragraph{Perfect identification}

Assume that \(P\) is absolutely continuous with respect to \(P_0\) on the true rare-valid set and amplifies that set by a constant likelihood factor \(\alpha\):
\begin{equation}
\frac{dP}{dP_0}(\omega)=\alpha,
\qquad \omega\in V_\delta,
\end{equation}
with \(P_0(V_\delta)=\delta\) and \(\alpha\delta\le 1\). Outside \(V_\delta\), \(P\) is renormalized so that it remains a probability distribution.

Substituting into Definition~2 gives
\begin{equation}
\mathcal I_\delta
=
\frac{P(V_\delta)-P_0(V_\delta)}{\delta}
=
\frac{\alpha\delta-\delta}{\delta}
=
\alpha-1.
\end{equation}
Thus, at fixed resolution \(\delta\),
$ \alpha=\mathcal I_\delta+1.$ If the limit defining \(\mathcal I\) exists and the amplification factor has a corresponding limiting value, then \(\alpha=\mathcal I+1\) in that limit.

The local log-likelihood entropy bookkeeping associated with an amplified rare-valid trajectory is
\cgathers{
S_{\rm ideal}^{\rm loc}
=
-k_B \log \frac{dP}{dP_0}
=
-k_B\log\alpha
=
-k_B\log(\mathcal I_\delta+1).
}
This is a local likelihood-accounting term, not a complete entropy balance without a specified physical protocol.

\paragraph{Imperfect identification}

Let the agent identify an approximate rare-valid set \(V'_\delta\) instead of \(V_\delta\). For the main theorem we analyze the conservative-identification case in which the estimated set contains the true rare-valid set,
\cgather{
    V_\delta \subseteq V'_\delta.
    \label{eq:conservative_identification}
}
This isolates false-positive cost and excludes false negatives. Define
\cgather{
    E_\delta = V'_\delta\setminus V_\delta,
    \qquad
    p_{\rm err}=P_0(E_\delta).
}
Assume that the agent amplifies \(V'_\delta\) by the same likelihood factor \(\alpha\):
\begin{equation}
\frac{dP}{dP_0}(\omega)=\alpha,
\qquad \omega\in V'_\delta.
\end{equation}
Require the normalization condition
\cgather{
    \alpha(\delta+p_{\rm err})\leq 1,
    \label{eq:imperfect_normalization_condition}
}
so that the amplified mass assigned to \(V'_\delta\) remains compatible with a probability law. Under \eqref{eq:conservative_identification}, \(P(V_\delta)=\alpha\delta\), hence \(\alpha=\calI_\delta+1\), and \(P(E_\delta)=\alpha p_{\rm err}\).

False-positive correction depends on the physical protocol used to store, test, correct, or erase erroneous assignments. We therefore keep the cost explicit and report both the expected per-trial term and the version normalized per amplified true rare-valid trajectory.

\begin{assumption}[Error-resolution protocol]
\label{ass:error_resolution_protocol}
For a false-positive region \(E_\delta\) with baseline mass
\(p_{\mathrm{err}}=P_0(E_\delta)\), the protocol used to resolve,
correct, or erase amplified erroneous assignments has entropy cost
\(k_B c(p_{\mathrm{err}})\) per unit amplified false-positive mass, where \(c(p)\geq 0\). The baseline-surprisal Landauer bookkeeping protocol corresponds to
\(c(p)=\log(1/p)\). This convention charges erroneous assignments
according to their rarity under the passive baseline \(P_0\), not
according to their mass after amplification. Other physical
implementations may induce different cost functions.
\end{assumption}

\begin{theorem}[Conditional imperfect rare-set identification accounting]
\label{thm:conditional_imperfect_penalty}
Let \(V_\delta\) be the true rare-valid set with \(P_0(V_\delta)=\delta\), and let \(V'_\delta\) be the agent's estimated rare-valid set. Suppose \(V_\delta\subseteq V'_\delta\), let
\[
    p_{\mathrm{err}}=P_0(V'_\delta\setminus V_\delta),
\]
assume \(\dd P/\dd P_0=\alpha\) on \(V'_\delta\), and assume \(\alpha(\delta+p_{\mathrm{err}})\leq1\). Under Assumption~\ref{ass:error_resolution_protocol}, the expected false-positive overhead per trial is
\cgather{
    \overline{\Delta S}_{\mathrm{err}}
    \triangleq 
    \alpha k_B p_{\mathrm{err}}\,c(p_{\mathrm{err}}).
    \label{eq:DeltaSerr_general}
}
The corresponding expected protocol-adjusted bookkeeping per trial is
\cgather{
    \overline S_{\mathrm{imperfect}}
    =
    -\alpha\delta k_B\log\alpha
    +
    \overline{\Delta S}_{\mathrm{err}}.
    \label{eq:imperfect_expected_general}
}
Equivalently, normalizing by the amplified true rare-valid mass \(P(V_\delta)=\alpha\delta\), the protocol-adjusted local bookkeeping per amplified true rare-valid trajectory is
\cgather{
    S_{\mathrm{imperfect}}^{\mathrm{loc}}
    =
    -k_B\log\alpha
    +
    \frac{p_{\mathrm{err}}}{\delta}k_B c(p_{\mathrm{err}}).
    \label{eq:imperfect_local_general}
}
For the baseline-surprisal Landauer bookkeeping protocol \(c(p)=\log(1/p)\),
\cgather{
    \overline{\Delta S}_{\mathrm{err}}
    =
    \alpha k_B p_{\mathrm{err}}
    \log\frac{1}{p_{\mathrm{err}}},
    \label{eq:DeltaSerr_baseline_landauer}
}
and
\cgather{
    S_{\mathrm{imperfect}}^{\mathrm{loc}}
    =
    -k_B\log\alpha
    +
    \frac{p_{\mathrm{err}}}{\delta}k_B
    \log\frac{1}{p_{\mathrm{err}}}.
    \label{eq:imperfect_alpha}
}
Using \(\alpha=\calI_\delta+1\), this becomes
\cgathers{
    S_{\mathrm{imperfect}}^{\mathrm{loc}}
    =
    -k_B\log(\calI_\delta+1)
    +
    \frac{p_{\mathrm{err}}}{\delta} k_B
    \log\frac{1}{p_{\mathrm{err}}}.
}
The expected version is
\mltlne{
    \overline S_{\mathrm{imperfect}}
    =
    -(\calI_\delta+1)\delta k_B\log(\calI_\delta+1)
    \\+
    (\calI_\delta+1)k_Bp_{\mathrm{err}}
    \log\frac{1}{p_{\mathrm{err}}}.
}
\end{theorem}
\noindent\emph{Proof.} See Appendix~\ref{app:proofs}.

\subsection{Link to recursive self-modeling resolution}
The dependence of \(p_{\rm err}\) on recursive model resolution is a statistical learning question, not a thermodynamic law. We therefore state rare-set stability as a hypothesis. In the conservative-identification case, the inferred set covers the true rare-valid set and only the excess false-positive mass shrinks with model fidelity.

Let \(V_\delta^{(k)}\) denote the rare-valid set inferred from the \(k\)-level recursive model \(r_B^{(k)}\). Define
\cgather{
p_{\rm err}^{(k)}
=
P_0\!\left(V_\delta^{(k)}\setminus V_\delta\right).
}
Let \(\varepsilon_k\) be the intervention-relevant prediction error from
Assumption~\ref{ass:model_to_control}.
\begin{hypothesis}[Rare-set stability under model convergence]\label{hyp:rare_set_stability}
There exists a modulus of continuity \(\rho_\delta\), with \(\rho_\delta(\varepsilon)\to 0\) as \(\varepsilon\to 0\), such that
\begin{equation}
p_{\rm err}^{(k)}
\le
\rho_\delta(\varepsilon_k).
\end{equation}
\end{hypothesis}
This hypothesis is nontrivial: small KL model error need not imply small false-positive mass on a rare set unless the rare-valid boundary is stable. A concrete finite-resolution sufficient condition is given in Appendix~\ref{app:rare_set_stability}; it shows that Hypothesis~\ref{hyp:rare_set_stability} follows from uniform validity-score convergence together with a margin bound on the rare-valid boundary.

\begin{corollary}[Conditional recursive self-modeling entropy efficiency]
\label{cor:conditional_entropy_efficiency}
Assume the baseline-surprisal Landauer bookkeeping protocol and
Hypothesis~\ref{hyp:rare_set_stability}. Suppose
\(V_\delta\subseteq V_\delta^{(k)}\),
\((\calI_\delta+1)(\delta+p_{\mathrm{err}}^{(k)})\leq1\),
\(p_{\mathrm{err}}^{(k)}<e^{-1}\), and
\(\rho_\delta(\varepsilon_k)<e^{-1}\). Define the expected excess bookkeeping penalty
\cgather{
    \overline{\Delta S}_{\mathrm{err}}^{(k)}
    \triangleq 
    \overline S_k
    +
    (\calI_\delta+1)\delta k_B\log(\calI_\delta+1).
    \label{eq:DeltaSerr_k_def}
}
Then
\mltlne{
    0
    \leq
    \overline{\Delta S}_{\mathrm{err}}^{(k)}
    =
    (\calI_\delta+1)k_B
    p_{\mathrm{err}}^{(k)}
    \log\frac{1}{p_{\mathrm{err}}^{(k)}}
    \\\leq
    (\calI_\delta+1)k_B
    \rho_\delta(\varepsilon_k)
    \log\frac{1}{\rho_\delta(\varepsilon_k)}.
    \label{eq:DeltaSerr_k_bound}}
Equivalently, normalized per amplified true rare-valid trajectory,
\mltlne{
    0
    \leq
    \Delta s_{\mathrm{err}}^{(k)}
    \triangleq 
    S_k^{\mathrm{loc}}
    +
    k_B\log(\calI_\delta+1)
    \\
    =
    \frac{k_B}{\delta}
    p_{\mathrm{err}}^{(k)}
    \log\frac{1}{p_{\mathrm{err}}^{(k)}}
    \leq
    \frac{k_B}{\delta}
    \rho_\delta(\varepsilon_k)
    \log\frac{1}{\rho_\delta(\varepsilon_k)}.
    \label{eq:Delta_s_err_k_bound}
}
Consequently,
\cgathers{
    \varepsilon_k\to 0
    \Longrightarrow\quad
    p_{\mathrm{err}}^{(k)}\to 0
    \Longrightarrow\quad
    \overline{\Delta S}_{\mathrm{err}}^{(k)}\to 0,
    \Delta s_{\mathrm{err}}^{(k)}\to 0,
}
and therefore
\cgather{
    S_k^{\mathrm{loc}}
    \to
    -k_B\log(\calI_\delta+1)
}
from above.
\end{corollary}
\noindent\emph{Proof.} See Appendix~\ref{app:proofs}.

Together, model-to-control stability and rare-set stability give the intended chain: recursive fidelity improves the controlled law, improves rare-set identification, and reduces false-positive bookkeeping waste. Amplification power and rare-set resolution remain distinct: a powerful but poorly resolved controller can amplify wrong futures, while a high-resolution but weakly actuating model can identify rare-valid futures without making them likely.

\section{Intelligence Calculations}

\subsection{Intelligence of Maxwell's Demons}
\label{sec:demon}

Maxwell's demon is the canonical limiting case of realized thermodynamic intelligence. In level-relative terms, the demon carries an effectively perfect simulation of the relevant gas microstate and the consequences of its gate action. It observes, stores, and acts to separate particles by velocity. Under passive dynamics, a temperature gradient is exponentially unlikely; under demon-assisted dynamics, the trajectory can become likely or deterministic. The demon is not a free entropy machine, but an idealized upper bound: near-perfect rare-valid fidelity, target identification, and actuation at the measured level.

Let \(\Delta S>0\) denote the magnitude of a local entropy reduction in the gas. For a matched entropy-producing trajectory, the fluctuation-theorem scale gives
\cgather{
    \frac{P_0(+\Delta S)}{P_0(-\Delta S)}
    \sim
    \exp\left(\frac{\Delta S}{k_B}\right).
}
Equivalently,
\cgather{
    P_0(-\Delta S)
    \sim
    P_0(+\Delta S)
    \exp\left(-\frac{\Delta S}{k_B}\right).
}
If the positive-entropy counterpart has order-one probability at the chosen coarse-graining, then the entropy-reducing event has passive probability on the scale \(\exp(-\Delta S/k_B)\). An ideal demon that realizes it with probability near one therefore has amplification
\cgather{
    \calI+1
    \sim
    \exp\left(\frac{\Delta S}{k_B}\right),
    \qquad
    \log_{10}(\calI+1)
    =
    \frac{\Delta S/k_B}{\log 10}.
    \label{eq:fluctuation_demon_scale}
}

\paragraph{Illustrative entropy-reduction calculation}

Suppose \(\Delta S=10^{-19}\,\mathrm{J/K}\). Since \(k_B=1.380649\times 10^{-23}\,\mathrm{J/K}\), we obtain
\cgather{
    \frac{\Delta S}{k_B}
    \approx
    \frac{10^{-19}}{1.380649\times 10^{-23}}
    \approx
    7243.
}
Therefore, at an \(O(1)\) positive-counterpart coarse-graining,
\cgathers{
    P_0(-\Delta S)
    \sim
    e^{-7243}
    \approx
    10^{-3146},
    \qquad
    \calI+1
    \sim
    e^{7243}
    \approx
    10^{3146}.
}
Thus an ideal demon that deterministically realizes this trajectory has a thermodynamic-intelligence scale on the order of \(\calI\sim 10^{3146}\). This should be read as a fluctuation-theorem scale calculation, not as a complete microscopic probability model for a particular gas protocol.

\paragraph{Velocity-selection demon}
Consider a classical ideal gas of \(N\) particles at temperature \(T\). Let the demon select particles whose kinetic energy exceeds a threshold
$    \epsilon    =     \frac{m v_c^2}{2k_B T}$. For a three-dimensional Maxwell--Boltzmann gas, the dimensionless kinetic energy \(x=mv^2/(2k_BT)\) follows a Gamma distribution with density
\cgather{
    f(x)
    =
    \frac{2}{\sqrt{\pi}}x^{1/2}e^{-x}.
}
The fraction of particles above threshold is
\cgathers{
    q(\epsilon)
    =
    \Prob(x>\epsilon)
    =
    \frac{\Gamma(3/2,\epsilon)}{\Gamma(3/2)}
    =
    \operatorname{erfc}(\sqrt{\epsilon})
    +
    \frac{2}{\sqrt{\pi}}\sqrt{\epsilon}e^{-\epsilon}.
}
The conditional mean dimensionless energy above threshold is
\cgather{
    \Exp[x\mid x>\epsilon]
    =
    \frac{\Gamma(5/2,\epsilon)}{\Gamma(3/2,\epsilon)}.
}
The excess dimensionless kinetic energy above the thermal mean \(3/2\) is
\cgather{
    \phi(\epsilon)
    =
    \frac{\Gamma(5/2,\epsilon)}{\Gamma(3/2,\epsilon)}
    -
    \frac{3}{2}.
}
The following is a local \(Q/T\)-scale estimate, not a full entropy accounting. Interpreting the selected particles' excess kinetic energy as the heat scale sorted at ambient temperature \(T\), the dimensionless local entropy-reduction proxy is
\cgather{
    \frac{|\Delta S|}{k_B}
    \approx
    \frac{N}{2}q(\epsilon)\phi(\epsilon).
    \label{eq:demon_entropy}
}
Here the factor \(1/2\) corresponds to the idealized one-shot protocol in which the demon acts on one half of the chamber. Therefore
\cgather{
    \log_{10}(\calI+1)
    \approx
    \frac{N}{2\log 10}q(\epsilon)\phi(\epsilon).
    \label{eq:demon_logI}
}

\paragraph{Concrete numerical instance (ambient nitrogen)}

To make the velocity-selection estimate explicit, consider nitrogen-like air at \(T=300\,\mathrm{K}\), with molecular mass \(m_{\mathrm{N_2}}\approx 4.65\times 10^{-26}\,\mathrm{kg}\). The cutoff speed corresponding to the dimensionless kinetic-energy threshold \(\epsilon\) is
\cgather{
    v_c(\epsilon)
    =
    \sqrt{\frac{2k_B T\epsilon}{m_{\mathrm{N_2}}}}.
}
At \(T=300\,\mathrm{K}\), this gives cutoff speeds of approximately \(597,731,844,\) and \(944\,\mathrm{m/s}\) for \(\epsilon=2,3,4,\) and \(5\), respectively. The excess kinetic energy sorted by the demon in the one-shot half-chamber protocol is
\cgathers{
    Q_{\mathrm{excess}}
    \approx
    \frac{N}{2}q(\epsilon)\phi(\epsilon)k_B T\\
    \Rightarrow 
    \frac{Q_{\mathrm{excess}}}{T}
    \approx
    k_B\frac{N}{2}q(\epsilon)\phi(\epsilon),
    \text{ and }
    \log_{10}(\calI+1)
    \approx
    \frac{Q_{\mathrm{excess}}}{k_B T\log 10}.
}
At one atmosphere and \(300\,\mathrm{K}\), an ideal gas has density \(2.44\times10^{25}\,\mathrm{m}^{-3}\), so \(1\,\mathrm{mm}^3\) contains \(N=2.44\times10^{16}\) particles. Table~\ref{tab:demon_velocity_numbers} reports the selected-particle count \(N_{\mathrm{sel}}\), sorted excess-energy scale \(Q_{\mathrm{excess}}\), and \(L=\log_{10}(\calI+1)\).

\begin{table}[!t]
\centering
\caption{Concrete velocity-selection demon estimates for nitrogen-like air at \(T=300\,\mathrm{K}\), using the corrected three-dimensional Maxwell--Boltzmann tail. \(N_{\mathrm{sel}}\) is the selected-particle count in the one-shot half-chamber protocol, \(Q_{\mathrm{excess}}\) is the local sorted excess kinetic-energy scale for \(1\,\mathrm{mm}^3\) of air, and \(L=\log_{10}(\calI+1)\). Scaled columns report the indicated common multipliers. Values are local amplification-scale estimates before measurement, memory, control, and erasure costs.}
\label{tab:demon_velocity_numbers}
\small
\setlength{\tabcolsep}{2.5pt}
\begin{tabular}{@{}ccccccccc@{}}
\toprule
\(\epsilon\) &
\begin{tabular}{@{}c@{}}\(v_c\)\\ \((\mathrm{m/s})\)\end{tabular} &
\(q\) &
\(\phi\) &
\begin{tabular}{@{}c@{}}\(N_{\mathrm{sel}}^{100}\)\end{tabular} &
\(L_{100}\) &
\begin{tabular}{@{}c@{}}\(N_{\mathrm{sel}}^{1\,\mathrm{mm}^3}\)\\\((10^{14})\)\end{tabular} &
\begin{tabular}{@{}c@{}}\(Q_{\mathrm{excess}}^{1\,\mathrm{mm}^3}\)\\\((10^{-6}\,\mathrm{J})\)\end{tabular} &
\begin{tabular}{@{}c@{}}\(L^{1\,\mathrm{mm}^3}\)\\\((10^{14})\)\end{tabular}
\\
\midrule
2 & 597 & 0.261  & 1.652 & 13.1  & 9.38 & 31.9 & 21.8 & 22.9 \\
3 & 731 & 0.112  & 2.615 & 5.58  & 6.34 & 13.6 & 14.8 & 15.5 \\
4 & 844 & 0.0460 & 3.593 & 2.30  & 3.59 & 5.61 & 8.35 & 8.76 \\
5 & 944 & 0.0186 & 4.578 & 0.928 & 1.85 & 2.27 & 4.30 & 4.50 \\
\bottomrule
\end{tabular}
\end{table}

The scale is extreme because the same selection rule acts on many microscopic degrees of freedom. For \(N=100\), \(\log_{10}(\calI+1)\approx1.85\)--\(9.38\). In \(1\,\mathrm{mm}^3\) of air, it acts on \(10^{14}\)--\(10^{15}\) eligible particles; the sorted excess energy is only \(O(10^{-5})\,\mathrm{J}\), but corresponds to roughly \(10^{15}\) thermal-scale selections.

Scaling from \(1\,\mathrm{mm}^3\) to \(1\,\mathrm{cm}^3\) multiplies \(N_{\mathrm{sel}}\), \(Q_{\mathrm{excess}}\), and \(L\) by \(10^3\), leaving \(v_c\), \(q(\epsilon)\), and \(\phi(\epsilon)\) unchanged. The local entropy reduction must still be compensated by measurement, memory, control, and erasure costs; the table gives idealized local amplification capacity, not a second-law violation.

\subsection{Symbolic Generation: Human and LLM Text}
\label{sec:symbolic}

Many intelligent systems emit symbolic sequences, including human speech or writing and LLM token streams. Even when direct thermodynamic work on the environment is unobserved, symbolic output provides an observable trajectory. The symbolic case therefore gives a finite empirical version of the rare-valid framework, using the information-theoretic language of entropy rate, coding, and sequence likelihood~\citep{shannon1948mathematical,coverthomas2006elements}.

Let \(X_{1:n}=(X_1,\ldots,X_n)\) be a symbolic sequence over a finite alphabet \(\Sigma\). For a prompt or task context \(y\), let \(P_0(\cdot\mid y)\) denote a baseline symbolic distribution on \(\Sigma^n\), and let \(P_G(\cdot\mid y)\) denote the distribution induced by generator \(G\). The baseline may be an \(n\)-gram model, a low-order Markov model, a prompt-independent language model, a lower-capacity reference generator, or a fixed decoding policy. Let \(V_n(y)\subseteq\Sigma^n\) be the valid outputs for context \(y\). Validity is task-dependent---syntactic well-formedness, semantic coherence, factual consistency, entailment, executable correctness, biological admissibility, or task relevance---and must be measurable under both \(P_0(\cdot\mid y)\) and \(P_G(\cdot\mid y)\).

\subsubsection{Symbolic rare-valid amplification}

A rare-valid symbolic set must be defined by baseline mass, not merely by a pointwise likelihood cutoff. Fix \(0<\delta<1\). Let \(R_{\delta,n}(y)\subseteq V_n(y)\) be a valid set whose baseline probability is \(\delta\):
\cgather{
    P_0(R_{\delta,n}(y)\mid y)
    =
    \sum_{x\in R_{\delta,n}(y)}P_0(x\mid y)
    =
    \delta.
    \label{eq:sym_rare_mass}
}
When exact equality is not attainable because \(\Sigma^n\) is discrete, one may use the nearest attainable mass, use \(P_0(R_{\delta,n}(y)\mid y)\leq \delta\), or randomize inclusion of a boundary element to obtain exact mass. A natural construction is to choose \(R_{\delta,n}(y)\) from the lowest-baseline-probability valid outputs until total baseline mass \(\delta\) is reached.

The symbolic thermodynamic intelligence of generator \(G\), at length \(n\), context \(y\), and rare-valid mass \(\delta\), is
\mltlne{
    \calI_{\delta,n}(G\mid y)
    =
    \frac{
    P_G(R_{\delta,n}(y)\mid y)
    -
    P_0(R_{\delta,n}(y)\mid y)
    }
    {
    P_0(R_{\delta,n}(y)\mid y)
    }
   \\ =
    \frac{P_G(R_{\delta,n}(y)\mid y)}{\delta}-1.
    \label{eq:sym_I}
}
Equivalently,
\cgather{
    P_G(R_{\delta,n}(y)\mid y)
    =
    \delta\left(1+\calI_{\delta,n}(G\mid y)\right).
    \label{eq:sym_mass_lift}
}
Thus \(\calI_{\delta,n}=0\) when the generator matches the baseline on the rare-valid set; positive values indicate amplification and negative values suppression. For a distribution \(\mu\) over prompts or tasks,
\cgather{
    \calI_{\delta,n}^{\mu}(G)
    =
    \Exp_{y\sim\mu}
    \left[
    \calI_{\delta,n}(G\mid y)
    \right].
    \label{eq:sym_population_I}
}
The symbolic definition inherits the binary coarse-graining bound from the path-space theory. Let
\cgather{
    p_G(y)
    =
    P_G(R_{\delta,n}(y)\mid y)
    =
    \delta\left(1+\calI_{\delta,n}(G\mid y)\right).
}
Then data processing for KL divergence under the binary partition \(\{R_{\delta,n}(y),R_{\delta,n}(y)^c\}\) gives
\cgather{
    D_{\mathrm{KL}}
    \left(
    P_G(\cdot\mid y)\,\|\,P_0(\cdot\mid y)
    \right)
    \geq
    d(p_G(y)\,\|\,\delta),
    \label{eq:sym_binary_kl}\\
    \text{ where } d(p\,\|\,\delta)
    =
    p\log\frac{p}{\delta}
    +
    (1-p)\log\frac{1-p}{1-\delta}\notag.
}
Equivalently,
\mltlne{
    D_{\mathrm{KL}}
    \left(
    P_G(\cdot\mid y)\,\|\,P_0(\cdot\mid y)
    \right)
    \\\geq
    d\left(
    \delta(1+\calI_{\delta,n}(G\mid y))
    \,\|\,\delta
    \right).
    \label{eq:sym_kl_I}
}
Thus symbolic rare-valid amplification requires measurable divergence from the baseline symbolic law.

\subsubsection{Validity-weighted operational estimator}

Empirical validity is often graded. Let \(v(x,y)\in[0,1]\) be a calibrated validity score and define
\cgather{
    w_{\delta,n}(x,y)
    =
    v(x,y)\,
    \mathbf{1}_{\{0<P_0(x\mid y)\leq q_{\delta}(y)\}},
}
where \(q_{\delta}(y)\) is chosen so that the baseline weighted mass is
\cgather{
    Z_{0,\delta}(y)
    =
    \sum_{x\in\Sigma^n}
    w_{\delta,n}(x,y)P_0(x\mid y)
    >
    0.
}
The validity-weighted symbolic intelligence is then
\cgather{
    \calI_{\delta,n}^{\mathrm{val}}(G\mid y)
    =
    \frac{
    \sum_{x\in\Sigma^n}
    w_{\delta,n}(x,y)
    \left(P_G(x\mid y)-P_0(x\mid y)\right)
    }
    {Z_{0,\delta}(y)}.
    \label{eq:sym_valid_weighted}
}
This reduces to the hard rare-valid definition when \(v\) is an indicator of \(R_{\delta,n}(y)\), and it penalizes entropy inflation because rare invalid outputs receive little or no weight.

For a concrete task family, let \(D_+\) contain valid prompt--output pairs and \(D_-\) invalid, corrupted, or adversarial outputs. Train a calibrated classifier \(c_\theta(x,y)\in[0,1]\) estimating validity. For threshold \(\tau\), define
\cgather{
    V_{n,\tau}(y)
    =
    \{x\in\Sigma^n:c_\theta(x,y)\geq \tau\}.
}
Given \(P_0(\cdot\mid y)\), choose a rare-tail threshold \(q_\delta(y)\) so that
\cgathers{
    R_{\delta,n,\tau}(y)
    =
    \{x\in\Sigma^n:
    c_\theta(x,y)\geq\tau,\;
    0<P_0(x\mid y)\leq q_\delta(y)
    \},
}
with
\cgather{
    P_0(R_{\delta,n,\tau}(y)\mid y)
    =
    \delta,
    \label{eq:sym_operational_mass}
}
up to the finite-alphabet boundary convention. The corresponding hard-threshold symbolic intelligence is
\cgather{
    \calI_{\delta,n,\tau}(G\mid y)
    =
    \frac{
    P_G(R_{\delta,n,\tau}(y)\mid y)
    }
    {\delta}
    -1.
    \label{eq:sym_operational_I}
}
Here \(c_\theta\) may combine grammaticality, entailment, and task relevance. For executable reasoning, validity can instead be defined by a parser, unit tests, proof checker, or verifier.

\subsubsection{Temperature as a negative control}
Let \(G_T\) denote an LLM decoded at temperature \(T\). The entropy of the generated symbolic distribution is
\cgather{
    H_n(G_T\mid y)
    =
    -\sum_{x\in\Sigma^n}
    P_{G_T}(x\mid y)
    \log P_{G_T}(x\mid y).
}
The validity-weighted rare-tail score is
\cgathers{
    \calI_{\delta,n}^{\mathrm{val}}(G_T\mid y)
    =
    \frac{
    \sum_{x\in\Sigma^n}
    w_{\delta,n}(x,y)
    \left(P_{G_T}(x\mid y)-P_0(x\mid y)\right)
    }
    {Z_{0,\delta}(y)}.
}
Entropy and rare-valid lift need not be monotone. Increasing temperature may raise \(H_n(G_T\mid y)\) while shifting mass into invalid or incoherent strings. Therefore
\cgather{
    \calI_{\delta,n}^{\mathrm{val}}(G_T\mid y)
    \not\equiv
    H_n(G_T\mid y).
}
The empirical prediction is an intermediate optimum: low temperature is valid but generic, high temperature is rare but often invalid, and the validity-weighted score peaks when outputs are both rare under \(P_0\) and valid.

\subsubsection{Set-level lift and symbolic scale calculations}

For the hard rare-valid set \(R_{\delta,n}(y)\), define the set-level log-lift
\cgather{
    \lambda_{\delta,n}(G\mid y)
    =
    \log
    \frac{
    P_G(R_{\delta,n}(y)\mid y)
    }
    {
    P_0(R_{\delta,n}(y)\mid y)
    }.
}
Since \(P_0(R_{\delta,n}(y)\mid y)=\delta\), this gives the exact identity
\cgather{
    \lambda_{\delta,n}(G\mid y)
    =
    \log\left(1+\calI_{\delta,n}(G\mid y)\right).
    \label{eq:set_lift_identity}
}
Equivalently,
\cgather{
    \calI_{\delta,n}(G\mid y)
    =
    \exp\left(\lambda_{\delta,n}(G\mid y)\right)-1.
    \label{eq:I_from_set_lift}
}
This set-level identity replaces an informal pointwise average log-lift. The latter,
\cgather{
    \Exp_{x\sim P_G(\cdot\mid R_{\delta,n}(y),y)}
    \left[
    \log
    \frac{P_G(x\mid y)}{P_0(x\mid y)}
    \right]
}
is generally unequal to \(\lambda_{\delta,n}(G\mid y)\), except when \(P_G(x\mid y)/P_0(x\mid y)\) is nearly constant over \(R_{\delta,n}(y)\). For sequence length \(n\), define the per-symbol structured log-lift in bits by
\cgather{
    \Delta\ell_{\delta,n}(G\mid y)
    =
    \frac{
    \lambda_{\delta,n}(G\mid y)
    }
    {n\log 2}.
    \label{eq:structured_log_lift}
}
Then
\cgather{
    \log_{10}\left(\calI_{\delta,n}(G\mid y)+1\right)
    =
    n\,\Delta\ell_{\delta,n}(G\mid y)\,\log_{10}2.
    \label{eq:symbolic_logI}
}
This is a structured log-lift decomposition, not an ordinary entropy-rate statement. The latter becomes relevant only after an additional finite-resolution AEP approximation has been specified.

\subsubsection{Sentence-scale human and AI estimates}
\label{sec:symbolic_human_ai_estimates}

We now instantiate the symbolic calculation at the sentence scale. Louwerse's simplified combinatorial estimate~\citep{louwerse2021sentences} gives a finite-resolution ensemble of interpretable English sentences of roughly \(3\)--\(20\) words on the order of
\cgather{
    N_V\triangleq |V_n|\approx 5\times10^{21}.
    \label{eq:valid_sentence_count}
}
This is not a universal linguistic constant; it is a usable baseline cardinality for the present finite-resolution calculation. To define a comparable rare-valid target, let \(G_n^\star\subset V_n\) denote sentence-scale strings that satisfy an additional high-quality human predicate: literary force, explanatory compression, originality, memorability, or comparable semantic/aesthetic force. We estimate its cardinality by taking a broad human canon of \(B_\star\sim2\times10^3\) high-quality long-form works and \(S_\star\sim5\times10^3\) sentence-scale units per work,
\cgather{
    N_G\triangleq |G_n^\star|
    \approx
    B_\star S_\star
    \sim
    (2\times10^3)(5\times10^3)
    =
    10^7.
    \label{eq:great_sentence_count}
}
The scale of this assumption is conservative relative to large public-domain corpora; the Standardized Project Gutenberg Corpus, for example, contains more than \(5\times10^4\) books and more than \(3\times10^9\) word tokens~\citep{gerlach2020standardized}. We treat \(N_G=10^6\)--\(10^8\) as a sensitivity range.

Taking the passive symbolic baseline to be approximately uniform over \(V_n\), the baseline mass of the exemplary human-quality set is
\cgather{
    \delta_\star
    =
    \frac{N_G}{N_V}
    \approx
    \frac{10^7}{5\times10^{21}}
    =
    2\times10^{-15}.
    \label{eq:human_delta_star}
}
If \(q_G=P_G(G_n^\star)\) is the probability that generator \(G\) lands in this target set under the specified task condition, then
\cgather{
    \calI_G+1
    =
    \frac{q_G}{\delta_\star}
    =
    q_G\,\frac{N_V}{N_G}.
    \label{eq:q_symbolic_lift}
}
For an expert human process conditioned on producing exemplary sentence-scale text, \(q_H\approx1\), giving
\cgather{
    \calI_H+1
    \approx
    \frac{5\times10^{21}}{10^7}
    =
    5\times10^{14}.
    \label{eq:human_symbolic_I}
}
Equivalently,
\cgather{
    L_H=\log_{10}(\calI_H+1)=14.699\\
    \Lambda_H
    =
    \log_{10}(L_H+1)
    =
    1.196.
    \label{eq:human_symbolic_lambda}
}

For the LLM estimate we use only one machine generator and one human reference corpus: Gutenberg prose and GPT-5 long-form prose. The entropy-rate estimates are generated by the self-contained procedure in Appendix~\ref{app:gutenberg_gpt5_entropy}. Both corpora are mapped to a common \(27\)-symbol alphabet, consisting of the letters \(a\)--\(z\) and space, with punctuation, digits, and non-ASCII characters removed. The resulting mean entropy rates are
\cgather{
    H_H=0.77\ \mathrm{bits/character}
    \quad\hbox{(for Gutenberg prose)},
    \\
    H_{\mathrm{GPT5}}=0.74\ \mathrm{bits/character}.
    \label{eq:gutenberg_gpt5_entropy_values}
}
Since Louwerse's valid-sentence estimate is a sentence-scale count, the entropy-rate correction must use a character-scale length. We take a 20-word sentence-scale unit to have \(n_\star=100\) characters after the same coarse-graining. The correction below is an AEP-style support-size approximation, not an exact finite-length theorem. More precisely, we write
\cgather{
    \log_2\frac{q_{\mathrm{GPT5}}}{q_H}
    =
    n_\star(H_{\mathrm{GPT5}}-H_H)+\rho_{n_\star},
    \label{eq:entropy_support_correction}
}
where \(\rho_{n_\star}\) collects finite-length typical-set error, boundary effects of the rare-valid set, and the fact that entropy-rate support size is only a proxy for overlap with the exemplary-output subset. The numerical value below sets \(\rho_{100}=0\), and should be read as a central order-of-magnitude estimate. With \(n_\star=100\), this gives
\cgather{
    \frac{q_{\mathrm{GPT5}}}{q_H}
    \approx
    2^{100(0.74-0.77)}
    =2^{-3}=0.125.
    \label{eq:gpt5_q_ratio_point}
}
Combining \eqref{eq:human_symbolic_I} and \eqref{eq:gpt5_q_ratio_point} gives the central estimate
\cgather{
    \calI_{\mathrm{GPT5}}+1
    \approx
    (5\times10^{14})
    2^{100(0.74-0.77)}
    =
    6.25\times10^{13}.
    \label{eq:gpt5_symbolic_I_entropy}
}
Thus
\cgather{
    L_{\mathrm{GPT5}}
    =
    \log_{10}(6.25\times10^{13})
    =13.796
    \\    \Lambda_{\mathrm{GPT5}}
    =
    \log_{10}(L_{\mathrm{GPT5}}+1)
    =1.170.
    \label{eq:gpt5_symbolic_estimate}
}
Since \(H_{\mathrm{GPT5}}<H_H\), the sentence-scale choice \(n_\star=100\) is conservative relative to longer independently composable symbolic units: increasing \(n_\star\) would decrease the GPT-5 estimate, provided the entropy-rate gap persists and the rare-valid construction is consistently extended.

\begin{table*}[!ht]
\caption{Numerical thermodynamic-intelligence scale sorted from small to large. The table reports the approximate rare-valid lift \(\calI\) and the stabilized double-log scale \(\Lambda=\log_{10}(\log_{10}(\calI+1)+1)\). Rows are ordered by the midpoint of the underlying \(L=\log_{10}(\calI+1)\) range when a range is reported. Symbolic entries are central finite-resolution estimates; the GPT-5 entry sets the finite-length AEP correction \(\rho_{100}=0\). Demon entries are local amplification scales before measurement, memory, control, and erasure costs.}
\label{tab:numerical_scale}
\footnotesize
\setlength{\tabcolsep}{3.5pt}
\renewcommand{\arraystretch}{1.08}
\begin{center}
\begin{tabular}{@{}llll@{}}
\toprule
\tc{0.25\textwidth}{\textbf{Example / regime}} & \tc{0.40\textwidth}{\textbf{Calculation basis}} & \tc{0.17\textwidth}{\(\boldsymbol{\calI}\)} & \tc{0.10\textwidth}{\(\boldsymbol{\Lambda}\)}\\
\midrule
\tc{0.25\textwidth}{Passive matter / passive gas} & \tc{0.40\textwidth}{Baseline dynamics, \(P=P_0\).} & \tc{0.17\textwidth}{\(0\)} & \tc{0.10\textwidth}{\(0\)}\\[3pt]
\tc{0.25\textwidth}{Narrow fixed-feedback controller} & \tc{0.40\textwidth}{Constant rare-valid lift, \(P(V_\delta)=\alpha\delta\), with \(\alpha=2\)--\(10^2\).} & \tc{0.17\textwidth}{\(1\)--\(99\)} & \tc{0.10\textwidth}{\(0.114\)--\(0.477\)}\\[3pt]
\tc{0.25\textwidth}{Repeated dynamic controller} & \tc{0.40\textwidth}{Sequential binary lift over \(7\)--\(10\) controlled stages, \(\calI+1\approx2^7\)--\(2^{10}\).} & \tc{0.17\textwidth}{\(1.27\times10^2\)--\(1.02\times10^3\)} & \tc{0.10\textwidth}{\(0.493\)--\(0.603\)}\\[3pt]
\tc{0.25\textwidth}{Sparse velocity-selection demon} & \tc{0.40\textwidth}{Maxwell--Boltzmann velocity selection with \(N=100\) particles and \(\epsilon=2\)--\(5\).} & \tc{0.17\textwidth}{\(7.0\times10^1\)--\(2.4\times10^9\)} & \tc{0.10\textwidth}{\(0.455\)--\(1.016\)}\\[3pt]
\tc{0.25\textwidth}{GPT-5 symbolic generation} & \tc{0.40\textwidth}{Central sentence-scale rare-valid lift with \(N_V=5\times10^{21}\), \(N_G=10^7\), \(n_\star=100\), \(H_H=0.77\), \(H_{\mathrm{GPT5}}=0.74\) bits/character, and \(\rho_{100}=0\).} & \tc{0.17\textwidth}{\(6.25\times10^{13}\)} & \tc{0.10\textwidth}{\(1.170\)}\\[3pt]
\tc{0.25\textwidth}{Expert human symbolic generation} & \tc{0.40\textwidth}{Same sentence-scale calculation with \(q_H\approx1\): \(\calI_H+1=N_V/N_G\).} & \tc{0.17\textwidth}{\(5.0\times10^{14}\)} & \tc{0.10\textwidth}{\(1.196\)}\\[3pt]
\tc{0.25\textwidth}{Maxwell demon with \(\Delta S=10^{-19}\,\mathrm{J/K}\)} & \tc{0.40\textwidth}{Fluctuation-theorem scale with \(\Delta S/k_B=7242.97\).} & \tc{0.17\textwidth}{\(\sim10^{3146}\)} & \tc{0.10\textwidth}{\(3.498\)}\\[3pt]
\tc{0.25\textwidth}{Velocity-selection demon in \(1\,\mathrm{mm}^3\) air} & \tc{0.40\textwidth}{Maxwell--Boltzmann velocity selection with \(N=2.44\times10^{16}\) and \(\epsilon=2\)--\(5\).} & \tc{0.17\textwidth}{\(10^{4.50\times10^{14}}\)--\(10^{2.29\times10^{15}}\)} & \tc{0.10\textwidth}{\(14.653\)--\(15.360\)}\\
\bottomrule
\end{tabular}
\end{center}
\end{table*}

\subsubsection{Algorithmic-complexity interpretation}

The symbolic construction has an algorithmic-statistics interpretation, but not a direct estimator. Kolmogorov complexity is uncomputable and machine-dependent up to additive constants~\citep{li2008kolmogorov}; the useful point is structural: algorithmic statistics separates the description of a model class from the index of an object inside it~\citep{gacs2001algorithmic}.

Let \(K(x)\) denote the prefix Kolmogorov complexity of a finite string \(x\). A two-part code describes \(x\) through a finite set or model class \(S\ni x\):
\cgather{
    K(x)\lesssim K(S)+\log |S|.
}
Here \(K(S)\) describes the regularity class and \(\log|S|\) indexes \(x\) within it. Thus high-complexity strings need not be noise; they may lie in rich valid classes. The target is not raw unpredictability, but probability mass assigned to rare strings that remain valid under task constraints.

Levin's coding theorem relates universal a priori probability \(m(x)\) to Kolmogorov complexity \citep{levin1974laws}:
\cgather{
    K(x)=-\log m(x)+O(1).
    \label{eq:levin_coding}
}
This gives an algorithmic analogue of symbolic rarity. If \(M\) is a universal semimeasure or computable approximation, normalize it over \(\Sigma^n\) by
\cgather{
    M_n(x)=\frac{M(x)}{\sum_{z\in\Sigma^n}M(z)}.
}
For an algorithmic rare-valid set \(A_{\delta,n}(y)\subseteq V_n(y)\) satisfying \(M_n(A_{\delta,n}(y))=\delta\), define
\cgather{
    \calI_{\delta,n}^{K}(G\mid y)
    =
    \frac{P_G(A_{\delta,n}(y)\mid y)}{\delta}-1.
    \label{eq:algorithmic_I}
}
Then the same binary KL bridge gives
\cgather{
    D_{\mathrm{KL}}\!\left(P_G(\cdot\mid y)\,\|\,M_n\right)
    \geq
    d\!\left(\delta(1+\calI_{\delta,n}^{K}(G\mid y))\,\|\,\delta\right).
    \label{eq:algorithmic_kl_bridge}
}
For a specified task distribution and baseline, larger symbolic thermodynamic intelligence means more probability mass on valid strings that are rare under that baseline. Empirical comparison therefore requires finite baselines, validity functions, and prompt distributions.

\section{Discussion}
\label{sec:discussion}
The central insight in this work is that perceived intelligence is measurable by what a system does to the probability distribution over possible futures. The architectural claim is that intelligence requires recursive self-simulation: a system models a world in which it is itself an acting component, evaluates action-conditioned futures, and uses that model to select interventions. The operational claim is that this architecture becomes measurable as rare-valid probability lift: amplification of futures that were unlikely under a passive baseline but remain valid under the constraints of the domain. The main mathematical claim is that these two ideas are not merely associated. Under bounded amplification, high rare-valid lift requires high rare-valid fidelity in the system's self-simulation, and high fidelity is nearly sufficient when an effective amplifying policy is available.

This perspective is adjacent to, but distinct from, several existing formalisms. Legg--Hutter intelligence measures expected reward over a universal distribution of computable environments, whereas the present quantity measures probability lift of a specified rare-valid path set relative to a specified passive law~\citep{legg2007universal}. Chollet's ARC framework emphasizes skill-acquisition efficiency and abstraction from sparse examples, whereas the present framework asks what path-law operation such success corresponds to once the baseline, validity criterion, and resolution are fixed~\citep{chollet2019measure}. Free-energy and active-inference formulations describe perception and action through variational free-energy minimization; here the measured object is not free energy itself but the induced reweighting of rare-valid trajectories~\citep{friston2010free}. The closest thermodynamic relative is semantic information, in which information is meaningful when it is causally necessary for a system to maintain viability under counterfactual interventions~\citep{kolchinsky2018semantic}. In contrast, rare-valid lift measures how much an induced law amplifies valid low-baseline-probability futures, and Theorems~\ref{thm:rv_fidelity_necessity} and~\ref{thm:rv_fidelity_sufficiency} connect that amplification to rare-valid simulation fidelity. Thus the present contribution is not rarity, validity, self-modeling, or feedback thermodynamics in isolation~\citep{sagawa2010generalized,parrondo2015thermodynamics}, but their combination into a level-relative path-measure definition with explicit fidelity bounds.

This formulation makes intelligence a level-relative measurement. A claim of intelligence is not made relative to an inaccessible absolute reality, but relative to a specified level of description, baseline path law, validity criterion, and observational resolution. At level \(k\), realized thermodynamic intelligence means that the system actually changes the level-\(k\) path law. Thermodynamic intelligence potential for level \(k\) is computed inside a level-\((k+1)\) simulation of level \(k\), where counterfactual actions can be evaluated even if they are not implemented at level \(k\). This distinction separates what a system can identify in simulation from what it actually makes more probable in the measured world.

The numerical examples calibrate the measure rather than define a taxonomy. Table~\ref{tab:numerical_scale} reports the stabilized double-log scale \(\Lambda=\log_{10}(\log_{10}(\calI+1)+1)\), which allows passive systems, feedback controllers, symbolic generators, and idealized information engines to be placed on the same probability-lift scale. Passive matter has zero lift by construction. Simple feedback produces modest lift; repeated dynamic control compounds small gains; symbolic generators can amplify valid low-baseline-probability sequences; and Maxwell-demon-like systems occupy the high end because microscopic information is used to select rare thermodynamic trajectories. The \(1\,\mathrm{mm}^3\) velocity-selection demon is large because the same selection rule is applied across \(10^{14}\)--\(10^{15}\) eligible particles. These are local amplification scales before full measurement, memory, control, and erasure costs are paid.

The formal results separate the ingredients of the theory. Theorems~\ref{thm:rv_fidelity_necessity} and~\ref{thm:rv_fidelity_sufficiency} provide the central bridge from recursive self-simulation to thermodynamic intelligence: rare-valid fidelity is necessary under bounded amplification and nearly sufficient with effective simulated actuation. Lemma~\ref{lem:rare_valid_binary_kl} shows that rare-valid amplification entails path-measure divergence from the passive baseline. Theorem~\ref{thm:path_deviation} and Proposition~\ref{prop:recursive_signature_convergence} relate path-law changes to coarse-grained thermodynamic signatures. Theorem~\ref{thm:conditional_imperfect_penalty} gives a protocol-dependent bookkeeping penalty for imperfect rare-set identification. Together, these results distinguish simulation fidelity, amplification power, thermodynamic accounting, and implementation.

Empirical use of \(\calI_\delta\) requires explicit choices of baseline, validity criterion, level of description, and trajectory resolution. These choices are not defects of the framework; they are the conditions under which the measurement is meaningful. Poor baselines can inflate or suppress measured lift, and continuous high-dimensional rare sets require statistical regularity conditions such as finite partitions, margin assumptions, or large-deviation structure. Natural testbeds include symbolic generation with executable or semantic checkers, closed-loop control with known passive dynamics, biological sequence evolution under viability constraints, and synthetic Maxwell-demon-like information engines with explicit measurement and erasure costs.

\section{Conclusion}
\label{sec:conclusion}

Intelligence can be treated as recursive self-simulation made observable through lawful amplification of rare-valid futures. A system is intelligent, in this sense, when it models a world containing itself, evaluates action-conditioned futures, and shifts probability mass toward futures that were rare under a specified passive baseline but remain valid. The central result is that high lift cannot be obtained from randomness or actuation alone: under bounded amplification it requires high rare-valid self-simulation fidelity, and with effective simulated actuation that fidelity yields lift near the actuation-limited optimum. The framework is level-relative, thermodynamically accounted, and applicable across passive systems, feedback controllers, Maxwell-demon-like information engines, and symbolic generators once the level, baseline law, validity criterion, trajectory resolution, and induced probability shift are specified.

\section*{Data and Code Availability}
The data, parameter files, and scripts used to generate the numerical calibration results reported in this manuscript are available through Harvard Dataverse at \url{https://doi.org/10.7910/DVN/F5TGT3}. The accompanying public GitHub repository, \url{https://github.com/zeroknowledgediscovery/tme}, contains the reproducibility code used to generate the values in Fig.~\ref{figtom}, Table~\ref{tab:demon_velocity_numbers}, Table~\ref{tab:numerical_scale}, and Appendix~\ref{app:numerical_scales}. The GPT--human entropy-rate values used as symbolic inputs are not re-estimated in the TME repository; they are imported as documented input constants from the workflow in \url{https://github.com/zeroknowledgediscovery/nero}, with provenance specified in the repository metadata and with the estimation protocol described in Appendix~\ref{app:gutenberg_gpt5_entropy}. 

\begin{acknowledgments}
  This work was supported by the Defense Advanced Research Projects Agency (DARPA) under the MAGICS program, DARPA-EA-25-02-05-MAGICS-PA-025, Award No. HR0011-26-3-E016. The views, opinions, and conclusions expressed in this work are those of the author and do not necessarily represent the official position or policy of DARPA or the U.S. Government.
\end{acknowledgments}

\clearpage


\appendix

\section{Proofs and Technical Qualifications}
\label{app:proofs}

\subsection{Proof of Theorem~\ref{thm:rv_fidelity_necessity}}

\noindent\emph{Proof.}
For readability, write \(\widehat P_0\), \(\widehat P_\pi\), \(\widehat V_\delta\), \(\widehat A\), \(\widehat\delta\), and \(\widehat\Phi\) for the corresponding level-\((k+1\to k)\) quantities. Decompose the simulated rare-valid mass as
\cgather{
    \widehat P_\pi(\widehat V_\delta)
    =
    \widehat P_\pi(\widehat V_\delta\cap\widehat A)
    +
    \widehat P_\pi(\widehat V_\delta\setminus\widehat A).
}
Since \(\widehat P_\pi\ll\widehat P_0\), the likelihood-ratio assumptions give
\cgather{
    \widehat P_\pi(\widehat V_\delta)
    \leq
    \alpha_{\max}\widehat P_0(\widehat V_\delta\cap\widehat A)
    +
    \widehat P_0(\widehat V_\delta\setminus\widehat A).
}
By definition,
\cgather{
    \widehat P_0(\widehat V_\delta\cap\widehat A)
    =
    \widehat\delta\widehat\Phi,
    \qquad
    \widehat P_0(\widehat V_\delta\setminus\widehat A)
    =
    \widehat\delta(1-\widehat\Phi).
}
Therefore
\cgather{
    \widehat P_\pi(\widehat V_\delta)
    \leq
    \widehat\delta
    \left[1+(\alpha_{\max}-1)\widehat\Phi\right].
}
Substituting into the definition of \(\widehat{\calI}_\delta^{(k+1\to k)}(\pi)\) gives
\cgather{
    \widehat{\calI}_\delta^{(k+1\to k)}(\pi)
    =
    \frac{\widehat P_\pi(\widehat V_\delta)-\widehat\delta}{\widehat\delta}
    \leq
    (\alpha_{\max}-1)\widehat\Phi.
}
If \(\alpha_{\max}>1\) and \(\widehat{\calI}_\delta^{(k+1\to k)}(\pi)\geq I_0>0\), then Eq.~\eqref{eq:rv_necessity_bound} implies
\cgather{
    I_0
    \leq
    (\alpha_{\max}-1)\widehat\Phi,
}
and rearrangement gives Eq.~\eqref{eq:rv_necessity_i0}. \(\square\)

\subsection{Proof of Theorem~\ref{thm:rv_fidelity_sufficiency}}

\noindent\emph{Proof.}
Use the same abbreviations as in the previous proof. The lower likelihood-ratio assumptions give
\cgather{
    \widehat P_\pi(\widehat V_\delta)
    \geq
    \alpha_{\min}\widehat P_0(\widehat V_\delta\cap\widehat A)
    +
    \beta_{\min}\widehat P_0(\widehat V_\delta\setminus\widehat A).
}
Using
\(\widehat P_0(\widehat V_\delta\cap\widehat A)=\widehat\delta\widehat\Phi\)
and
\(\widehat P_0(\widehat V_\delta\setminus\widehat A)=\widehat\delta(1-\widehat\Phi)\),
we obtain
\cgather{
    \widehat P_\pi(\widehat V_\delta)
    \geq
    \widehat\delta
    \left[
    \alpha_{\min}\widehat\Phi
    +
    \beta_{\min}(1-\widehat\Phi)
    \right].
}
Therefore
\cgather{
    \widehat{\calI}_\delta^{(k+1\to k)}(\pi)
    =
    \frac{\widehat P_\pi(\widehat V_\delta)-\widehat\delta}{\widehat\delta}
    \geq
    \alpha_{\min}\widehat\Phi
    +
    \beta_{\min}(1-\widehat\Phi)
    -1.
}
If \(\widehat\Phi\geq1-\varepsilon\), then the right-hand side is minimized over \(\widehat\Phi\in[1-\varepsilon,1]\) at \(\widehat\Phi=1-\varepsilon\) whenever \(\alpha_{\min}\geq\beta_{\min}\). This yields
\cgather{
    \widehat{\calI}_\delta^{(k+1\to k)}(\pi)
    \geq
    \alpha_{\min}(1-\varepsilon)+\beta_{\min}\varepsilon-1\\
    =
    (\alpha_{\min}-1)-(\alpha_{\min}-\beta_{\min})\varepsilon.
}
\(\square\)

\subsection{Proof of Theorem~\ref{thm:path_deviation}}
\begin{proof}
Let
\[
    a=P(A_s^+),\quad b=P(A_s^-),\quad
    a_0=Q(A_s^+),\quad b_0=Q(A_s^-).
\]
Then
\begin{align}
    |\Delta_s(P,Q)|
    &=
    \left|
    \log\frac{a}{b}-\log\frac{a_0}{b_0}
    \right|\\
    &\leq
    |\log a-\log a_0|
    +
    |\log b-\log b_0|.
\end{align}
Since all four probabilities are at least \(m_s\), the logarithm is \(1/m_s\)-Lipschitz on this interval, so
\begin{equation}
    |\Delta_s(P,Q)|
    \leq
    \frac{1}{m_s}\left(|a-a_0|+|b-b_0|\right).
\end{equation}
The two event probabilities are components of the coarse-grained distributions induced by the partition containing \(A_s^+\), \(A_s^-\), and the complement. Hence
\begin{equation}
    |a-a_0|+|b-b_0|
    \leq 2\,\TV(P,Q).
\end{equation}
By Pinsker's inequality,
\begin{equation}
    \TV(P,Q)
    \leq
    \sqrt{\frac{1}{2}\KL(P\,\|\,Q)}.
\end{equation}
Therefore,
\begin{equation}
    |\Delta_s(P,Q)|
    \leq
    \frac{\sqrt{2}}{m_s}
    \sqrt{\KL(P\,\|\,Q)}.
\end{equation}
Taking \(P=P_B\) and \(Q=P_0\) gives \eqref{eq:controlled_passive_deviation_bound}.
\end{proof}

\subsection{Proof of Proposition~\ref{prop:recursive_signature_convergence}}
\begin{proof}
By definition,
\caligns{
    \Delta_s(P_B^{(k)},P_0)
    -
    \Delta_s(P_B^\star,P_0)
    &=
    \log\frac{P_B^{(k)}(A_s^+)}{P_B^{(k)}(A_s^-)}
    -
    \log\frac{P_B^\star(A_s^+)}{P_B^\star(A_s^-)}\\
    &=
    \Delta_s(P_B^{(k)},P_B^\star).
}
Applying Theorem~\ref{thm:path_deviation} with \(P=P_B^{(k)}\) and \(Q=P_B^\star\) gives
\cgathers{
    \left|
    \Delta_s(P_B^{(k)},P_0)
    -
    \Delta_s(P_B^\star,P_0)
    \right|
    \leq
    \frac{\sqrt{2}}{m_s}
    \sqrt{\KL(P_B^{(k)}\,\|\,P_B^\star)}.}
Using \eqref{eq:model_to_control_stability} yields \eqref{eq:recursive_signature_convergence}.
\end{proof}

\subsection{Proof of Lemma~\ref{lem:rare_valid_binary_kl}}
\begin{proof}
Apply the data-processing inequality for KL divergence to the binary coarse-graining \(\{V_\delta,V_\delta^c\}\). The induced Bernoulli laws have success probabilities \(p=P(V_\delta)\) and \(\delta=P_0(V_\delta)\). Therefore
\[
    \KL(P\,\|\,P_0)
    \geq
    d(p\,\|\,\delta).
\]
Substituting \(p=\delta(1+\calI_\delta)\) gives \eqref{eq:Idelta_kl_bridge}.
\end{proof}

\subsection{Proof of Theorem~\ref{thm:conditional_imperfect_penalty}}
\begin{proof}
Because \(V_\delta\subseteq V'_\delta\) and \(\dd P/\dd P_0=\alpha\) on \(V'_\delta\), the true rare-valid mass is \(P(V_\delta)=\alpha\delta\). Hence \(\calI_\delta=(\alpha\delta-\delta)/\delta=\alpha-1\). The false-positive region \(E_\delta=V'_\delta\setminus V_\delta\) has baseline mass \(p_{\mathrm{err}}\), so its amplified controlled mass is \(P(E_\delta)=\alpha p_{\mathrm{err}}\). By Assumption~\ref{ass:error_resolution_protocol}, resolving this amplified erroneous mass costs \(k_Bc(p_{\mathrm{err}})\) per unit amplified false-positive mass. Therefore the expected overhead is \(\overline{\Delta S}_{\mathrm{err}}=\alpha k_B p_{\mathrm{err}}c(p_{\mathrm{err}})\).

The expected ideal bookkeeping contribution from the amplified true rare-valid mass is \(-\alpha\delta k_B\log\alpha\), giving \eqref{eq:imperfect_expected_general}. Dividing \eqref{eq:imperfect_expected_general} by \(P(V_\delta)=\alpha\delta\) gives the local per-amplified-true-rare-valid expression \eqref{eq:imperfect_local_general}. The Landauer forms follow by setting \(c(p)=\log(1/p)\), and the substitution \(\alpha=\calI_\delta+1\) follows from the conservative-identification identity above.
\end{proof}

\subsection{Proof of Corollary~\ref{cor:conditional_entropy_efficiency}}
\begin{proof}
For \(0<p<e^{-1}\), the function
$    g(p)=p\log\frac{1}{p}$ is increasing. Hypothesis~\ref{hyp:rare_set_stability} gives
\(p_{\mathrm{err}}^{(k)}\leq \rho_\delta(\varepsilon_k)\). Therefore,
\[
    p_{\mathrm{err}}^{(k)}
    \log\frac{1}{p_{\mathrm{err}}^{(k)}}
    \leq
    \rho_\delta(\varepsilon_k)
    \log\frac{1}{\rho_\delta(\varepsilon_k)}.
\]
Multiplying by \((\calI_\delta+1)k_B\) gives
\eqref{eq:DeltaSerr_k_bound}. Multiplying by \(k_B/\delta\) gives
\eqref{eq:Delta_s_err_k_bound}. Since
\(\rho_\delta(\varepsilon)\to0\) and \(p\log(1/p)\to0\) as
\(p\to0^+\), both excess bookkeeping penalties vanish as
\(\varepsilon_k\to0\). Hence \(S_k^{\mathrm{loc}}\) approaches the ideal
local bookkeeping value from above.
\end{proof}

\subsection{Additional protocol and rare-set qualifications}

The main text uses the conservative-identification case \(V_\delta\subseteq V'_\delta\) to isolate false-positive overhead and preserve the identity \(\alpha=\calI_\delta+1\). If false negatives are allowed, define
\cgather{
    t=P_0(V_\delta\cap V'_\delta),\qquad
    f=P_0(V'_\delta\setminus V_\delta),
    \qquad
    m=\delta-t.
}
If \(\dd P/\dd P_0=\alpha\) on \(V'_\delta\) and \(\dd P/\dd P_0=\beta\) outside \(V'_\delta\), with
\cgather{
    \beta=\frac{1-\alpha(t+f)}{1-t-f},
    \qquad
    P(V_\delta)=\alpha t+\beta(\delta-t),
}
then
\cgather{
    \calI_\delta=\frac{\alpha t+\beta(\delta-t)-\delta}{\delta}.
}
Thus \(\alpha=\calI_\delta+1\) is specific to the no-false-negative case.

The baseline-surprisal Landauer form in Eq.~\eqref{eq:DeltaSerr_baseline_landauer} charges false-positive assignments according to their rarity under \(P_0\). A controlled-distribution erasure protocol instead uses \(P(E_\delta)=\alpha p_{\mathrm{err}}\), giving, when \(\alpha p_{\mathrm{err}}<1\),
\[
    \alpha k_Bp_{\mathrm{err}}
    \log\frac{1}{\alpha p_{\mathrm{err}}}
    =
    \alpha k_Bp_{\mathrm{err}}
    \left(\log\frac{1}{p_{\mathrm{err}}}-\log\alpha\right).
\]
The protocol-dependent form in the main text keeps this choice explicit. Similarly, the path-deviation bound in Theorem~\ref{thm:path_deviation} is a moderate-event stability bound because of the factor \(1/m_s\); rare-valid amplification is handled separately through Lemma~\ref{lem:rare_valid_binary_kl}.

\subsection{A sufficient condition for rare-set stability}
\label{app:rare_set_stability}

Hypothesis~\ref{hyp:rare_set_stability} is a statistical regularity condition, not a thermodynamic law. Here we record a concrete finite-resolution setting in which it holds. Fix an observational partition \(\Pi_\eta\) and a rare-tail region \(R_{\delta,\eta}\), chosen under the passive law \(P_0\). Let \(s:\Pi_\eta\to\mathbb R\) be a cell-level validity score and let \(\tau\) be a validity threshold, so that the true rare-valid set at this resolution is
\cgather{
    V_{\delta,\eta}
    =
    \bigcup_{C\in\Pi_\eta:\, C\subseteq R_{\delta,\eta},\, s(C)\geq \tau} C .
}
Let \(s_k:\Pi_\eta\to\mathbb R\) be the validity score induced by the \(k\)-level recursive model, and define the inferred set
\cgather{
    V_{\delta,\eta}^{(k)}
    =
    \bigcup_{C\in\Pi_\eta:\, C\subseteq R_{\delta,\eta},\, s_k(C)\geq \tau} C .
}
Assume that the score error is uniformly controlled by the intervention-relevant model error:
\cgather{
    \sup_{C\subseteq R_{\delta,\eta}}
    |s_k(C)-s(C)|
    \leq c\,\varepsilon_k
    \label{eq:score_uniform_control}
}
for some constant \(c>0\). Also assume a boundary-margin condition: there is a modulus \(m_\delta(t)\), with \(m_\delta(t)\to0\) as \(t\to0\), such that
\cgather{
    P_0\!\left(
    \bigcup_{C\subseteq R_{\delta,\eta}:\, 0\leq \tau-s(C)\leq t} C
    \right)
    \leq m_\delta(t).
    \label{eq:rare_boundary_margin}
}
Then
\cgather{
    P_0\!\left(V_{\delta,\eta}^{(k)}\setminus V_{\delta,\eta}\right)
    \leq
    m_\delta(c\,\varepsilon_k).
    \label{eq:finite_partition_stability_bound}
}
Thus Hypothesis~\ref{hyp:rare_set_stability} holds with
\(\rho_\delta(\varepsilon)=m_\delta(c\varepsilon)\).

Indeed, if a cell \(C\subseteq R_{\delta,\eta}\) is a false positive, then \(s_k(C)\geq\tau\) but \(s(C)<\tau\). By Eq.~\eqref{eq:score_uniform_control},
\[
    \tau-s(C)
    \leq
    s_k(C)-s(C)
    \leq
    c\,\varepsilon_k .
\]
Hence every false-positive cell lies in the boundary band
\[
    \{C\subseteq R_{\delta,\eta}: 0\leq \tau-s(C)\leq c\,\varepsilon_k\},
\]
and Eq.~\eqref{eq:finite_partition_stability_bound} follows from the margin condition.

A particularly simple case occurs when the validity boundary has a positive finite-resolution margin: if there exists \(\gamma>0\) such that no cell in \(R_{\delta,\eta}\) satisfies \(0<|\tau-s(C)|\leq\gamma\), then \(m_\delta(t)=0\) for \(t<\gamma\). Consequently, whenever \(c\,\varepsilon_k<\gamma\), the false-positive mass is zero:
\[
    P_0\!\left(V_{\delta,\eta}^{(k)}\setminus V_{\delta,\eta}\right)=0.
\]
More generally, if the boundary band satisfies a polynomial margin bound \(m_\delta(t)\leq C_\delta t^a\) for constants \(C_\delta>0\) and \(a>0\), then
\[
    p_{\rm err}^{(k)}
    \leq
    C_\delta c^a \varepsilon_k^a .
\]
This gives an explicit modulus for Hypothesis~\ref{hyp:rare_set_stability}.

\section{Entropy-rate Estimation for text}
\label{app:gutenberg_gpt5_entropy}

This appendix gives the self-contained entropy-rate protocol used in Section~\ref{sec:symbolic_human_ai_estimates}. The goal is not to reproduce a full detector or model-comparison study, but only to obtain two character-level entropy-rate estimates on a common symbolic alphabet: a human prose reference and a GPT-5 long-form prose estimate.

All texts are converted to a fixed \(27\)-symbol alphabet,
\[
    \Sigma_{27}=\{a,b,\ldots,z,\text{space}\}.
\]
Text is lowercased; punctuation, digits, and non-ASCII characters are removed. The resulting character stream is treated as a finite sample path from an approximately stationary symbolic source. Entropy is reported in bits per character.

For a preprocessed sequence \(s_{1:N}\), we use a nonparametric probabilistic-finite-state-automaton entropy-rate estimator~\citep{CLx,CL12g}. For a substring-frequency threshold \(m\), substrings with fewer than \(m\) occurrences are excluded. Empirical next-symbol distributions are estimated from retained histories and histories inducing similar next-symbol laws are represented as states of an inferred finite-state source. If \(Q_m\) is the resulting state set, \(\widehat\pi_m(q)\) is the empirical state frequency, and \(\widehat p_m(a\mid q)\) is the empirical next-symbol law, the threshold-specific estimate is
\begin{equation}
    \widehat H^{(m)}
    =
    -\sum_{q\in Q_m}\widehat\pi_m(q)
      \sum_{a\in\Sigma_{27}}
      \widehat p_m(a\mid q)\log_2\widehat p_m(a\mid q).
\end{equation}
To reduce dependence on a single pruning threshold, the reported document-level estimate is the median across a fixed threshold grid,
\begin{equation}
    \widehat H
    =
    \mathrm{median}_{m\in\{m_1,\ldots,m_M\}}
    \widehat H^{(m)}.
\end{equation}
The estimator requires no language model, labels, or training corpus. The thresholding step removes poorly supported histories; the median aggregation stabilizes the estimate across admissible frequency cutoffs.

The human reference corpus consists of English Project Gutenberg long-form prose. Legal headers and boilerplate are removed, and texts shorter than \(150{,}000\) post-processed characters are excluded. This yields \(4341\) Gutenberg documents. The GPT-5 corpus consists of \(197\) long-form prose samples generated with a fixed narrative-prompt protocol using GPT-5 API access. Each sample is generated toward a target length of approximately \(150{,}000\) characters using repeated continuation calls under default sampling settings, with only maximum completion length controlled.

The cohort-level summary statistics used in the symbolic calculation are
\begin{center}
\scriptsize
\setlength{\tabcolsep}{2pt}
\begin{tabular}{@{}lcccc@{}}
\toprule
Source & mean \(\widehat H\) & median \(\widehat H\) & s.d. \(\widehat H\) & count\\
\midrule
GPT-5 long-form prose & 0.74 & 0.74 & 0.08 & 197\\
Project Gutenberg prose & 0.77 & 0.78 & 0.12 & 4341\\
\bottomrule
\end{tabular}
\end{center}
The main text uses the mean values \(H_{\mathrm{GPT5}}=0.74\) and \(H_H=0.77\) bits per character. The median values give the same qualitative ordering; using the Gutenberg median \(0.78\) instead would make the GPT-5 support correction smaller by an additional factor of \(2\) at \(n_\star=100\).

\section{Numerical Scale Calculations}
\label{app:numerical_scales}

This appendix supports the numerical values in Table~\ref{tab:numerical_scale}. The table does not report the raw lift \(\calI\), because the values span from \(0\) to powers such as \(10^{10^{15}}\). Instead it reports the compressed double-log scale
\begin{align}
    L &\triangleq  \log_{10}(\calI+1), \label{eq:appendix_L_def}\\
    \Lambda &\triangleq  \log_{10}(L+1). \label{eq:appendix_lambda_def}
\end{align}
The \(+1\) terms keep the scale finite at the passive baseline. For all nonzero large examples, \(\Lambda\) behaves as an ordinary \(\log\log\)-scale. When a row is a range, the endpoints of the displayed \(\Lambda\)-range are obtained by applying Eq.~\eqref{eq:appendix_lambda_def} to the endpoint values of \(L\). The ordering in Table~\ref{tab:numerical_scale} uses the midpoint of the underlying \(L\)-range.

\subsection{Passive baseline}

For a passive system, \(P=P_0\), so \(\calI=0\). Therefore
\cgather{
    L = \log_{10}(1)=0,  \text{ and }
    \Lambda = \log_{10}(1)=0.
}

\subsection{Fixed-feedback amplification}

Suppose a controller amplifies a rare-valid set by a constant likelihood factor \(\alpha\). Then
\begin{align}
    P(V_\delta) &= \alpha\delta,\\
    \calI_\delta
        &=\frac{\alpha\delta-\delta}{\delta}
          =\alpha-1,\\
    L &=\log_{10}(\calI_\delta+1)
       =\log_{10}\alpha .
    \label{eq:appendix_constant_lift}
\end{align}
For the fixed-feedback row we use \(\alpha=2\) to \(10^2\). Thus
\begin{align}
    \alpha=2:\quad
    &L=0.301,\quad \Lambda=0.114,\\
    \alpha=10^2:\quad
    &L=2.000,\quad \Lambda=0.477.
\end{align}
This gives the table entry \(\Lambda=0.114\)--\(0.477\).

\subsection{Repeated dynamic control}

Repeated control compounds set-level lift. If \(m\) stages have approximate lift factors \(\alpha_1,\ldots,\alpha_m\), then
\begin{align}
    \calI+1 &\approx \prod_{j=1}^{m}\alpha_j,\\
    L &\approx \sum_{j=1}^{m}\log_{10}\alpha_j .
    \label{eq:appendix_sequential_lift}
\end{align}
The repeated-control row uses seven to ten binary improvements:
\begin{equation}
    \calI+1\approx 2^7\text{--}2^{10}.
\end{equation}
Hence
\begin{align}
    m=7:\quad
    &L=7\log_{10}2=2.107,\quad \Lambda=0.493,\\
    m=10:\quad
    &L=10\log_{10}2=3.010,\quad \Lambda=0.603.
\end{align}
This gives \(\Lambda=0.493\)--\(0.603\).

\subsection{Sentence-scale symbolic lift}

For symbolic sequences, Eq.~\eqref{eq:symbolic_logI} gives
\begin{equation}
    L=n\,\Delta\ell\,\log_{10}2,
    \label{eq:appendix_symbolic_lift}
\end{equation}
where \(n\) is sequence length and \(\Delta\ell\) is the per-symbol structured set-level lift in bits. The sentence-scale calculation in Section~\ref{sec:symbolic_human_ai_estimates} instead begins from cardinalities.

Louwerse's finite-resolution combinatorial estimate~\citep{louwerse2021sentences} gives
\begin{equation}
    N_V\approx 5\times10^{21}
\end{equation}
valid or interpretable English sentence-scale strings. The exemplary human-quality target set is estimated as
\begin{equation}
    N_G
    \approx
    B_\star S_\star
    \sim
    (2\times10^3)(5\times10^3)
    =
    10^7.
\end{equation}
Thus
\begin{align}
    \delta_\star
    &=
    \frac{N_G}{N_V}
    =
    2\times10^{-15},\\
    \calI_H+1
    &\approx
    \frac{N_V}{N_G}
    =
    5\times10^{14}.
\end{align}
Therefore
\begin{align}
    L_H
    &=
    \log_{10}(5\times10^{14})
    =
    14.699,\\
    \Lambda_H
    &=
    \log_{10}(14.699+1)
    =
    1.196.
\end{align}
Using \(N_G=10^6\)--\(10^8\) gives \(\calI_H+1=5\times10^{15}\)--\(5\times10^{13}\) and \(\Lambda_H=1.223\)--\(1.167\).

For the entropy-rate-corrected GPT-5 calculation, we use an AEP-style support-size approximation with an explicit finite-length slack term,
\begin{equation}
    \log_2\frac{q_{\mathrm{GPT5}}}{q_H}
    =
    n_\star(H_{\mathrm{GPT5}}-H_H)+\rho_{n_\star}.
    \label{eq:appendix_entropy_ai_lift_slack}
\end{equation}
The term \(\rho_{n_\star}\) is not estimated here. It represents finite-length typical-set error, rare-valid boundary effects, and support-overlap error. This qualification matters because \(n_\star=100\) characters is sentence-scale, whereas AEP convergence is asymptotic. The point estimate in Table~\ref{tab:numerical_scale} sets \(\rho_{100}=0\):
\begin{align}
    2^{100(H_{\mathrm{GPT5}}-H_H)}
    &=
    2^{100(0.74-0.77)}
    =2^{-3}=0.125,\\
    \calI_{\mathrm{GPT5}}+1
    &\approx
    (5\times10^{14})2^{-3}
    =6.25\times10^{13},\\
    \Lambda_{\mathrm{GPT5}}
    &=
    \log_{10}\{\log_{10}(6.25\times10^{13})+1\}
    =1.170 .
\end{align}
Equivalently, the finite-length corrected expression is
\begin{equation}
    \calI_{\mathrm{GPT5}}+1
    \approx
    (5\times10^{14})2^{100(0.74-0.77)+\rho_{100}}.
\end{equation}
The compressed score \(\Lambda\) is relatively insensitive to moderate multiplicative changes, but the raw \(\calI\) should be interpreted only at order-of-magnitude precision. The entropy estimates are in bits per character under the \(27\)-symbol alphabet, so the sentence-scale length used here is \(n_\star=100\) characters, not \(20\) words.

\subsection{Fluctuation-theorem demon}

For an ideal demon that realizes an entropy-reducing trajectory of magnitude \(\Delta S\), the fluctuation-theorem scaling gives
\begin{equation}
    L=\log_{10}(\calI+1)
      =\frac{\Delta S/k_B}{\log 10}.
    \label{eq:appendix_fluctuation_demon}
\end{equation}
With \(\Delta S=10^{-19}\,\mathrm{J/K}\) and \(k_B=1.380649\times10^{-23}\,\mathrm{J/K}\),
\cgather{
    \frac{\Delta S}{k_B} = 7242.97,
    L = 3145.58, 
    \Lambda = \log_{10}(3146.58)=3.498.
}
The table rounds this to \(\Lambda=3.498\).

\subsection{Velocity-selection demon}

For the Maxwell--Boltzmann velocity-selection demon, Eq.~\eqref{eq:demon_logI} gives
\begin{equation}
    L\approx
    \frac{N}{2\log 10}\,
    q(\epsilon)\phi(\epsilon),
    \label{eq:appendix_velocity_demon}
\end{equation}
where \(N\) is the number of particles, \(q(\epsilon)\) is the fraction above the velocity threshold, and \(\phi(\epsilon)\) is the excess dimensionless kinetic energy above the thermal mean. The numerical inputs \(q(\epsilon)\), \(\phi(\epsilon)\), and \(L\) are listed in Table~\ref{tab:demon_velocity_numbers}.

For \(N=100\), the endpoint values over \(\epsilon=2\)--\(5\) are
\cgather{
    L=1.85\text{--}9.38, \text{ and }
    \Lambda
      =\log_{10}(L+1)
        =0.455\text{--}1.016.}
For \(1\,\mathrm{mm}^3\) of air at one atmosphere and \(300\,\mathrm{K}\), \(N\approx2.44\times10^{16}\). Table~\ref{tab:demon_velocity_numbers} gives
\begin{align}
    L&=4.50\times10^{14}
       \text{--}2.29\times10^{15},\\
    \Lambda&=14.653\text{--}15.360.
\end{align}
The \(1\,\mathrm{mm}^3\) demon therefore exceeds the single-event \(\Delta S=10^{-19}\,\mathrm{J/K}\) demon on this scale because it aggregates a very large number of microscopic velocity selections.

\end{document}